\documentclass{article} % For LaTeX2e
\usepackage{iclr2024_conference,times}

\usepackage{amsmath,amsfonts,bm}

\def\eqref#1{equation~\ref{#1}}
\def\1{\bm{1}}

\DeclareMathAlphabet{\mathsfit}{\encodingdefault}{\sfdefault}{m}{sl}
\SetMathAlphabet{\mathsfit}{bold}{\encodingdefault}{\sfdefault}{bx}{n}

\newcommand{\R}{\mathbb{R}}

\newcommand{\Var}{\mathrm{Var}}

\usepackage{graphicx}
\usepackage{hyperref}
\usepackage{url}

\DeclareMathOperator{\sgn}{sgn}

\newcommand{\nin}{{n_{\mathrm{in}}}}
\newcommand{\nout}{{n_{\mathrm{out}}}}
\newcommand{\ninin}{n_{\mathrm{in,in}}}
\newcommand{\ninout}{n_{\mathrm{in,out}}}
\newcommand{\noutin}{n_{\mathrm{out,in}}}
\newcommand{\noutout}{n_{\mathrm{out,out}}}
\newcommand{\nintwo}{n_{2-\mathrm{in}}}
\newcommand{\nouttwo}{n_{2-\mathrm{out}}}
\newcommand{\din}{d_{\mathrm{in}}}
\newcommand{\dout}{d_{\mathrm{out}}}
\newcommand{\Pin}{P_{\mathrm{in}}}
\newcommand{\Pout}{P_{\mathrm{out}}}
\newcommand{\pin}{p_{\mathrm{in}}}
\newcommand{\pout}{p_{\mathrm{out}}}
\usepackage{amsthm}
\usepackage{amssymb}
\usepackage{nicefrac} 
\usepackage{hyperref}
\usepackage{cleveref}
\usepackage{mathtools}
\usepackage{svg}

\newtheorem{lemma}{Lemma}
\newtheorem{theorem}{Theorem}

\newtheorem{remark}{Remark}
\usepackage[english]{babel}
\newcommand{\abs}[1]{\left| #1 \right|}

\title{Global minima, recoverability thresholds, and higher-order structure in GNNs}

\iclrfinalcopy
\author{%
  Drake B.~Brown\\
  Department of Mathematics\\
  Brigham Young University\\
  Provo, UT 84602\\
  \texttt{dbrown68@byu.edu}\\
  \And
  Trevor Garrity \\
  Department of Mathematics\\
  Brigham Young University \\
  Provo, UT 84602 \\\texttt{garrityt@byu.edu}
  \And
  Kaden Parker \\
  Department of Mathematics\\
  Brigham Young University \\
  Provo, UT 84602 \\\texttt{bobert11@byu.edu}
  \And
  Jason Oliphant \\
  Department of Mathematics\\
  Brigham Young University \\
  Provo, UT 84602 \\\texttt{jolipha2@byu.edu}
  \And
  Stone Carson \\
  Department of Mathematics\\
  Brigham Young University \\
  Provo, UT 84602 \\
  \texttt{stone@mathematics.byu.edu} \\
  \And
  Cole Hanson \\
  Department of Mathematics\\
  Brigham Young University \\
  Provo, UT 84602 \\
  \texttt{chanson@byu.edu} \\
  \And
  Zachary Boyd \\
  Department of Mathematics\\
  Brigham Young University \\
  Address \\
  \texttt{zachboyd@byu.edu} \\
}

\begin{document}

\maketitle

\begin{abstract}
We analyze the performance of graph neural network (GNN) architectures from the perspective of random graph theory. Our approach promises to complement existing lenses on GNN analysis, such as combinatorial expressive power and worst-case adversarial analysis, by connecting the performance of GNNs to typical-case properties of the training data.  First, we theoretically characterize the nodewise accuracy of one- and two-layer GCNs relative to the contextual stochastic block model (cSBM) and related models. We additionally prove that GCNs cannot beat linear models under certain circumstances. Second, we numerically map the recoverability thresholds, in terms of accuracy, of four diverse GNN architectures (GCN, GAT, SAGE, and Graph Transformer) under a variety of assumptions about the data. Sample results of this second analysis include:
heavy-tailed degree distributions enhance GNN performance,
GNNs can work well on strongly heterophilous graphs,
and SAGE and Graph Transformer can perform well on arbitrarily noisy edge data, but no architecture handled sufficiently noisy feature data well.
Finally, we show how both specific higher-order structures in synthetic data and the mix of empirical structures in real data have dramatic effects (usually negative) on GNN performance.

\end{abstract}
\section{Introduction}
    \label{sec:Intro}

Graph neural networks (GNNs) have achieved impressive success across many domains, including natural language processing~\citep{wu2022graph}, image representation learning \citep{adnan2020representation}, and perhaps most impressively in protein folding prediction \citep{jumper_highly_2021}. GNNs' success across these fields is due to their ability to harness non-Euclidean graph topology in the learning process~\citep{Xu19}. 
Despite the growing use of GNN architectures, we still grapple with a significant knowledge gap concerning the intricate relationship between the statistical structure of graph data and the nuanced behavior of these models. 

By aligning GNN designs with data distributions, we can not only unveil the underlying mechanics and behaviors of these models but also pave the way for architectures that intuitively resonate with inherent data patterns.

While significant focus has been directed towards homophily in the context of GNN performance \citep{maurya2021improving,Halcrow_2020,Zhu20}, other critical properties of graph data have remained relatively underexplored. Features such as degree distribution and mesoscale structure offer important insights into the behavior of networks. Similarly, despite the depth of theoretical advancements in graph modularity, including works such as the one by \citet{Abbe17}, there remains a sizable gap in their integration and applicability within the GNN domain. We seek to explore such properties to bridge this gap.

In this paper we: (1) fully characterize the nodewise accuracy of one- and two-layer GNNs satisfying certain assumptions, as well as proving that there exists a linear classifier that performs at least as well as any GCN over a graph drawn from a broad family, (2) report extensive numerical studies that map the degree to which edge and feature information contribute to overall performance across diverse models in a variety of random graph contexts, and (3) compare GNN performance on common benchmarks relative to matched synthetic data, demonstrating the dramatic (and usually negative) impact of various higher-order graph structures.

\section{Previous work}
    \label{sec:prev-work}

As part of this work we lay out theoretical bounds for GNN architectures. Some foundational work in our topic is as follows. ~\citet{font} investigated regimes in which the attention module in the Graph (GAT)~\citep{Veličković18} makes a meaningful difference in performance. Following this,~\citet{Baranwal22} proved theoretically that using graph convolutions expands the range where a vanilla neural network can correctly classify nodes.

~\citet{baranwal2022graph} discovered that linear classifiers on GNN embeddings generalize well to out of distribution data in stochastic block models. \citet{lu2022learning} characterized how well a GNN can separate communities on a two-class stochastic block model. Lastly,~\citet{revisiting_GNN} found that a Graph Convolutional Network (GCN) performs low pass filtering on the feature vectors and doesn't learn non-linear manifolds.

While many have attempted to understand models through the lens of specialized data, our approach offers a unique and deeper perspective on the subject. The monograph \citet{Abbe17} lays out the key mathematical findings related to SBMs as they relate to community detection. \citet{Karrer10} developed the \emph{degree-corrected SBM}, which allows for heavy-tailed degree distributions. Degree-corrected SBMs have been subjected to much analysis in recent years. For example,~\citet{gao_community_2018} derived asymptotic minimax risks for misclassification in degree-corrected SBMs. \citet{sbm_gnn} propose a variational autoencoder specifically designed to work with graphs generated by stochastic block models.
Lastly,~\citet{Deshpande18} proposed a \emph{contextual SBM} (cSBM) that generates feature data alongside the graph data. This was originally proposed to analyze specific properties of belief propagation~\citep{bickson_gaussian_2009}.

Our investigation presents a novel angle that bridges interplay between edge data and feature data.~\citet{binkiewicz_covariate-assisted_2017} explored how to use features to aid spectral clustering. \citet{yang_graph_2022}
 and \citet{arroyo_inference_2020} used edges and features that contain orthogonal information to better understand the relationship between the two.

While the influence of motifs or higher-order structures on GNN performance remains a hot area of exploration, our approach delves deeper into this pressing topic.

Works such as \citet{tu2020learning} have proposed using graphlets to aid in learning representations. Others have utilized hypergraphs to make better predictions \citep{huang2021unignn}.
Much of the work quantifying the expressive power of GNNs is achieved by relating GNNs to the classical Weisfeiler-Leman (WL) heuristic for graph isomorphism~\citep{wu_expressive_2022,huang_short_2021}. These have inspired corresponding GNN architectures that have increased distinguishing capabilities~\citep{hamilton_theoretical_2020}

\section{Background}
    \label{sec:background}

In this work, we first theoretically determine nodewise accuracy for certain one- and two-layer GNNs and identify cases where GNNs cannot outperform linear models. We map the performance of the Graph Convolutional Network ~\citep{Kipf17}, Graph SAGE~\citep{Hamilton17}, the Graph Attention Network~\citep{Veličković18}, and the Structure-Aware Transformer~\citep{Chen22} on several related random graph models related to the cSBM. We will also inject and remove higher-order structure in various contexts to see how GNN performance is affected. We now describe some of the random graph models and GNN architectures on which our analysis relies. Note, when referring to data generation methods we use the term \emph{generative models} while \emph{model} will refer to a trained GNN.

\subsection{Stochastic block models}
    \label{sec:SBM}

The stochastic block model is a random graph model that encodes node clusters (``classes'') in the graph topology. 
The presence or absence of each edge is determined by an independent Bernoulli draw with probability determined by the class identities of the nodes. We restrict attention to SBMs where all classes have the same size and uniform inter-class and intra-class probabilities. The parameters for such an SBM are: the total number of nodes $n$, the number of equally sized classes $k$, the intra-class edge probability $p_{\text{in}}$, and the inter-class edge probability $p_{\text{out}}$. While SBMs generate realistic clustering patterns, without further modification they exhibit a binomial degree distribution. To more closely model many realistic classes of data,~\citet{Karrer10} proposed the degree-corrected SBM, which can exhibit any degree distribution, notably heavy-tailed distributions.

In this paper, we represent edge similarity using an \emph{edge information parameter}, $\lambda$, which has the following relationship to $p_{\text{in}}$ and $p_{\text{out}}$:
\[p_{\text{in}}=\frac{d+\lambda \sqrt{d}}{n},\quad p_{\text{out}}=\frac{d-\lambda \sqrt{d}}{n},\] 
where $d$ is the expected average node degree.
Setting $\lambda=0$ yields identical inter- and intra-class edge probabilities, meaning the topology of the graph encodes no information about class labels. A positive $\lambda$ indicates that nodes of the same class are more likely to connect than nodes of different classes (homophily), while a negative $\lambda$ indicates the reverse relationship (heterophily).

To generate node attributes, \citet{Deshpande18} proposed the contextual SBM (cSBM), where features are drawn from Gaussian point clouds with mean at a specified distance $\mu$ from the origin. Features, $X$, are thus defined as $X(i)=\mu m_{v_i} + z_i$, where $z_i$ a standard normally distributed random variable, $v_i$ is the ground-truth class label of node $i$, and $m_{v_i}$ is the mean for class $v_i$. The means are chosen to be an orthogonal set. We can then vary the level of feature separability (feature information) by modifying $\mu$. Setting $\mu = 0$ makes node features indistinguishable across classes, while a large value of $\mu$ indicates high distinguishability. We thus refer to $\mu$ as the \emph{feature information parameter}.

\subsection{Graph neural networks}
    \label{sec:gnn-explanation}

As stated before, we analyze the performance of four diverse and influential architectures: GCN \cite{Kipf17}, SAGE \cite{Hamilton17}, GAT \cite{Veličković18}, and Graph-Transformer \cite{Chen22}. In our numerical work, we also assess the performance of a standard feedforward neural network and spectral clustering~\citep{Luxburg07}, which are useful points of comparison as they are agnostic to the graph and feature structures, respectively. Lastly we also use graph-tool~\citep{peixoto_graph-tool_2014} to evaluate feature agnostic performance on heterophilous graphs.

\section{Theoretical results}
\label{sec:theory}

We now derive analytically the performance of GNN architectures when the data-generating process is known. \Cref{sec:1layer} covers the one-layer case for a GCN architecture and cSBM-generated data, and \cref{sec:2layer} handles the two-layer case in for a more general class of GNN architecture as well as a broader class of generating processes. We introduce the following notation first: for a given node $i$, $\nin$ is the number of neighbors in the same class as $i$, and $\nout$ is the number of nodes in other classes. $\mathcal{N}(i)$ is the one-hop neighborhood of $i$. $v_i$ is the ground-truth class label of $i$. $\mathrm{erf}$ is the Gaussian error function. Both subsections assume a binary classification setting.

\subsection{Accuracy estimates for single-layer GCNs}\label{sec:1layer}
In the one-layer case, we assume the GNN is of the simple form $y(X) = \mathrm{sign}(AXW)$, that the final embedding is into $\R$, and that $A$ and $X$ are generated by a cSBM, with no self-loops (but see \cref{rem:self-loops}). We also make a slight modification to the cSBM setup so that the means are diametrically opposed rather than orthogonal. That is, 
\begin{equation*}
X(i) = 
\begin{cases}
    \mu m + z_i & \text{if node $i$ is in class $1$} \\
    -\mu m + z_i & \text{if node $i$ is in class $2$}.
    \end{cases}
\end{equation*}
This requires no loss of generality, since all choices of two means may be translated to fit this assumption. 
We then have the following theorem:
\begin{theorem}
    Under the preceding assumptions, we have
    \begin{enumerate}
        \item For each $i$, $(AXW)_i$ is distributed as 
        \[\underbrace{\mu(\nin - \nout)mW}_\text{neighborhood signal} + \underbrace{\left(\sum_{j \in \mathcal{N}(i)} {z_j}\right) W}_\text{noise}.\]
        \item If $W \ne 0$, the generalization accuracy, conditioned on the graph structure is, 
        \[P(y[X](i) = 1  \mid \nin,\nout, v_i = 1) = \frac{1}{2}\left(\mathrm{erf}\left(\frac{\mathrm{\mu}(\nin-\nout)}{\sqrt{2(\nin+\nout)}}\cos \theta\right)+1\right),\]
        where $\theta$ is the angle between $W$ and $m$.
        \item The maximum accuracy in the homophilous regime is achieved when $\theta=\pi$. In the heterophilous regime, $\theta=0$ is the maximizer.
    \end{enumerate}
    \label{thm:1layer}
\end{theorem}
\begin{proof}
    See~\cref{sec:1layer_appendix}.
\end{proof}

\begin{remark}
    Part three of this theorem shows that, in the one-layer case, optimal performance is achieved simply by aligning the learned parameters with the axis separating the means of the distributions. The proof consists largely of manipulations of the probability densities, together with calculus. A similar alignment result applies in the two-layer case, but in that case, the fastest way forward is to rely on the symmetries of the distribution and GNN, as shown below.
\end{remark}

\begin{remark}
\label{rem:self-loops}
    The analysis with self-loops is nearly identical, with the exception that it is possible that the maximizing parameters may possibly different in the extremely dense, slightly heterophilous case, but this is not the regime in which GNNs are typically used. See the proof for full details.
\end{remark}

\subsection{Analysis of two-layer GCNs}
\label{sec:2layer}

We define a \textit{2-class attributed random graph model} to be a probability space $(\Omega, P)$ of tuples $(G, i, v, X)$ where $G$ is a graph, $i$ is a node in $G$, and $v$ and $X$ are functions mapping each node in the graph to its class and its feature vector, respectively. That is,
\begin{align*}
    v &: G \rightarrow \{-1, 1\}\\
    X &: G \rightarrow \R^{m_\mathrm{feat}}
\end{align*}
As a notational convenience, $v(x)$ will denote the class of the node corresponding to a tuple $x \in \Omega$.

If $x = (G, i, v, X) \in \Omega$, we define the negation of $x$ to be the tuple $-x = (G, i, -v, -X)$. In other words, $x$ has the same graph with all of the classes and features negated. Similarly, we define the negation of a subset $F \subset \Omega$ to be
\[
-F = \{-x : x \in F\}
\]
We say a 2-class attributed random graph model is \textit{class-symmetric about the origin} if
\[
P(F) = P(-F)
\]
for all measurable $F \subset \Omega$. In other words, the distribution of graphs is symmetric about the origin respect to the two classes. A cSBM with an equal number of nodes in both classes satisfies this symmetry.

Furthermore, let $S$ be any subspace of $\R^{m_\mathrm{feat}}$ and let $P_S, R_S : \R^{m_\mathrm{feat}} \rightarrow \R^{m_\mathrm{feat}}$ be the linear maps projecting a vector onto $S$ and reflecting a vector across $S$, respectively. Then if $x = (G, i, v, X)$, we define
\[
R_S(x) = (G, i, v, R_S \circ X)
\]
and similarly, if $F \subset \Omega$ then we define
\[
R_S(F) = \{R_S(x) : x \in F\}
\]

We say that $\Omega$ is \textit{symmetric about $S$} if $P(R_S(F)) = P(F)$
for all measurable $F \subset \Omega$.

In general, a model $y$ on a 2-class attributed random graph model assigns to each $x \in \Omega$ a real number $y(x) \in \R$ that corresponds to the estimated probability that the node corresponding to $x$ is of class 1. More concretely, the predicted probability is given by $\sigma_s(y(x))$ where $\sigma_s : \R \rightarrow (0,1)$ is defined by,
\[
\sigma_s(z) = \frac{1}{1 + e^{-z}}.
\]
According to maximum likelihood learning, the cost function of the model $y$ is 
\[
C(y) = \mathbb{E}_{x \sim \Omega} [-\log P(v(x) : y(x))]
\]
where
\[
P(v(x) : y(x)) = \left\{\begin{array}{lr}
\sigma_s(y(x)) & v(x) = 1\\
1- \sigma_s(y(x)) & v(x) = -1.
\end{array} \right.
\]

A graph aggregation maps a graph and its features to a new set of features:
\[
\phi : (G, X) \rightarrow X'
\]
where $X' : G \rightarrow \R^l$ for some $l$. We say $\phi_G = \phi(G, \cdot)$. A linear aggregator (without bias), $\phi$, satisfies
\[
\phi_G(X_1 + X_2) = \phi_G(X_1) + \phi_G(X_2)
\]
for all graphs $G$ and features $X_1, X_2 : G \rightarrow \R^l$. A generalized 2-layer graph convolutional network (GCN) without bias is then given by
\[
y(x) = (\phi'_G \circ \sigma \circ \phi_G)[X](i)
\]
where $\phi$ and $\phi'$ are linear un-biased aggregators, $\phi'$ maps into $\R$, and $\sigma$ is the ReLU function. We define the linear approximation of $y$ to be
\[
L[y](x) = \frac{1}{2}(\phi'_G \circ \phi_G)[X](i)
\]
Notably, the non-linear ReLU is removed. Furthermore, we define the projection of this linear approximation onto $S$ to be
\[
P_S[L[y]](x) = \frac{1}{2}(\phi'_G \circ \phi_G \circ P_S)[X](i).
\]
If $\phi'$ and $\phi$ are both simply the classical right-multiplication by a weight matrix followed by summing the features of neighbors, then model $y$ becomes
\[
y(x) = \sum_{j \in \mathcal{N}(x)} \sigma \left(\sum_{k \in \mathcal{N}(j)} X(k)W \right) \cdot c
\]
where $W \in \R^{m_\mathrm{feat} \times p}$ and $c \in \R^p$ and $\mathcal{N}(x)$ denotes the set of neighbors of the node relating to $x$.

If $S = \R m$ for some unit vector $m \in \R^{m_\mathrm{feat}}$, as is the case with a cSBM, then
\[
P_S[L[y]](x) = \frac{1}{2}\sum_{j \in \mathcal{N}(x)}\sum_{k \in \mathcal{N}(j)} P_m(X(k))W \cdot c = K\sum_{j \in \mathcal{N}(x)}\sum_{k \in \mathcal{N}(j)} X(k) \cdot m
\]
for some $K \in \R$.

\begin{theorem}
    Let $\Omega$ be a 2-class attributed random graph model and let $y$ be any two-layer generalized GCN without bias on $\Omega$. If $\Omega$ is class-symmetric about the origin then,
    \[
    C(L[y]) \leq C[y].
    \]
    Furthermore, if $\Omega$ is symmetric about $S$ then,
    \[
    C(P_S[L[y]]) \leq C(L[y])
    \]
\end{theorem}

\begin{proof}
    See~\cref{appendix:2layer}. The main idea is to use the symmetries of the space together with the convexity of the objective to invoke Jensen's inequality.
\end{proof}

In light of the the preceding theorem, linear GCNs are optimal over the binary cSBM. Carefully analyzing the linear case, we obtain an explicit formula for the optimal accuracy of any GCN over cSBM data (with a remark afterward to explain the intuition behind several variables):
\begin{theorem}
    The linear model
    \[
    y(x) = K\sum_{j \in \mathcal{N}(x)} \sigma \left(\sum_{k \in \mathcal{N}(j)} X(k) \cdot m \right)
    \]
has accuracy
\[
\sum_{\nin, \nout, \nintwo, \nouttwo = 0}^\infty P(\nin, \nout, \nintwo, \nouttwo) \Phi\bigg(\psi\bigg(\sgn(K)\tfrac{\mu}{\sigma}, \nin, \nout, \nintwo, \nouttwo\bigg)\bigg)
\]
where $\Phi$ is the cdf of the standard normal distribution and the following definitions apply:
\begin{align*}
    &P(\nin, \nout, \nintwo, \nouttwo)\\
    &= p(\nin, \din)\cdot p(\nout, \dout)
    \cdot p(\nintwo, \din\nin + \dout\nout)\cdot p(\nouttwo, \dout\nin+\din\nout) \\
    &p(k,\lambda) = \frac{\lambda^k e^{-\lambda}}{k!}, \text{ and}\\
    &\psi(c, \nin, \nout, \nintwo, \nouttwo) = c\frac{1 + 3\nin - \nout + \nintwo - \nouttwo}{\sqrt{(\nin + \nout + 1)^2 + 4(\nin + \nout) + (\nintwo + \nouttwo)}}.
\end{align*}
\end{theorem}
\begin{proof}
    See~\cref{appendix:2layer}.
\end{proof}

\begin{remark}
    In the theorem, the indices $\nin$, $\nout$, $\nintwo$, and $\nouttwo$ refer to the number of distance 1 and 2 nodes with the same and the opposite class of the base node. The function $P$ represents the probability of the graph structure, while the rest of the formula is the accuracy of the given graph structure.
\end{remark}

\section{Empirical exploration of data regimes}

In~\cref{sec:experiment} and~\cref{sec:binary-gcn}, we present results from our simplest set of experiments in detail to illustrate the interplay between edges and features.
Then, in~\cref{sec:architecture_compare} we compare performance across each of the four architectures. 
Finally, we contrast how GNNs performed on degree-corrected and non-degree-corrected graphs in~\cref{sec:DC_GCN}. 

See also our full code online to extend this work to other architectures and parameter ranges:

\subsection{Experimental design}
    \label{sec:experiment}

To better understand how GNN architectures harness information embedded in the features or edges, we evaluated them across a variety of graphs.
Each of our architectures was comprised of one input layer, a hidden layer of size $16$ (with ReLU activation functions), and an output layer (with softmax). As baselines, we trained a feedforward neural network, with one hidden layer of size $16$, on the feature data. Our exploration also encompassed a variety of methods for feature-agnostic methods such as graph-tool~\citep{peixoto_graph-tool_2014}, Leidenalg (python package), Louvian (python package), and Spectral clustering~\citep{scikit-learn}. In doing so we found that spectral clustering worked the best for assortative graphs (edge information from [0,3]) and graphtool performed the best on dissasortative graphs (edge information from [-3,0)). 

We generated graph data using a cSBM with average degree $d = 10$; the number of nodes $n = 1,000$; the number of features $m_{\text{feat}}=10$; the number of classes $c=2$; and standard deviation of the Gaussian clouds $.2$. These hyperparameters were selected to be representative of a large variety of datasets without being too computationally expensive (specifically when using transformers). We observed that $1,000$ nodes was large enough to get statistical regularity and that using larger graphs (up to 40,000 nodes) didn't introduce major deviations. With these hyperparameters, we vary $\lambda$ (edge separation in cSBMs) between $-3$ and $3$ and vary feature separation (cloud distance from origin) from $0$ to $2$ to obtain $121\times200$ (how finely we discretized the interval) possible sets of graph data. This data ranges from being highly disassortative to highly assortative.

To train each architecture, we used an Adam optimizer (PyTorch) with a learning rate of $0.01$ for $400$ epochs (typically where the model ceased improving). 
We evaluated the final accuracy on a separate graph, with the same graph parameters to prevent overfitting. 

In addition to the class count of two, we ran the architectures across class counts of three, five, and seven each with both a degree-corrected case and a binomial case. As each test was averaged/maxed over $10$ trials, the number of tests totals $320$ different tests with $15,488,000$ accuracy scores generated (more than .25 petaflops used in total).

\subsection{Example: Binary node classification with graph transformer}
    \label{sec:binary-gcn}

\begin{figure}[htbp]
    \begin{center}
    % We need this or it takes up its own whole page
    \includegraphics[width=9\textwidth/10]{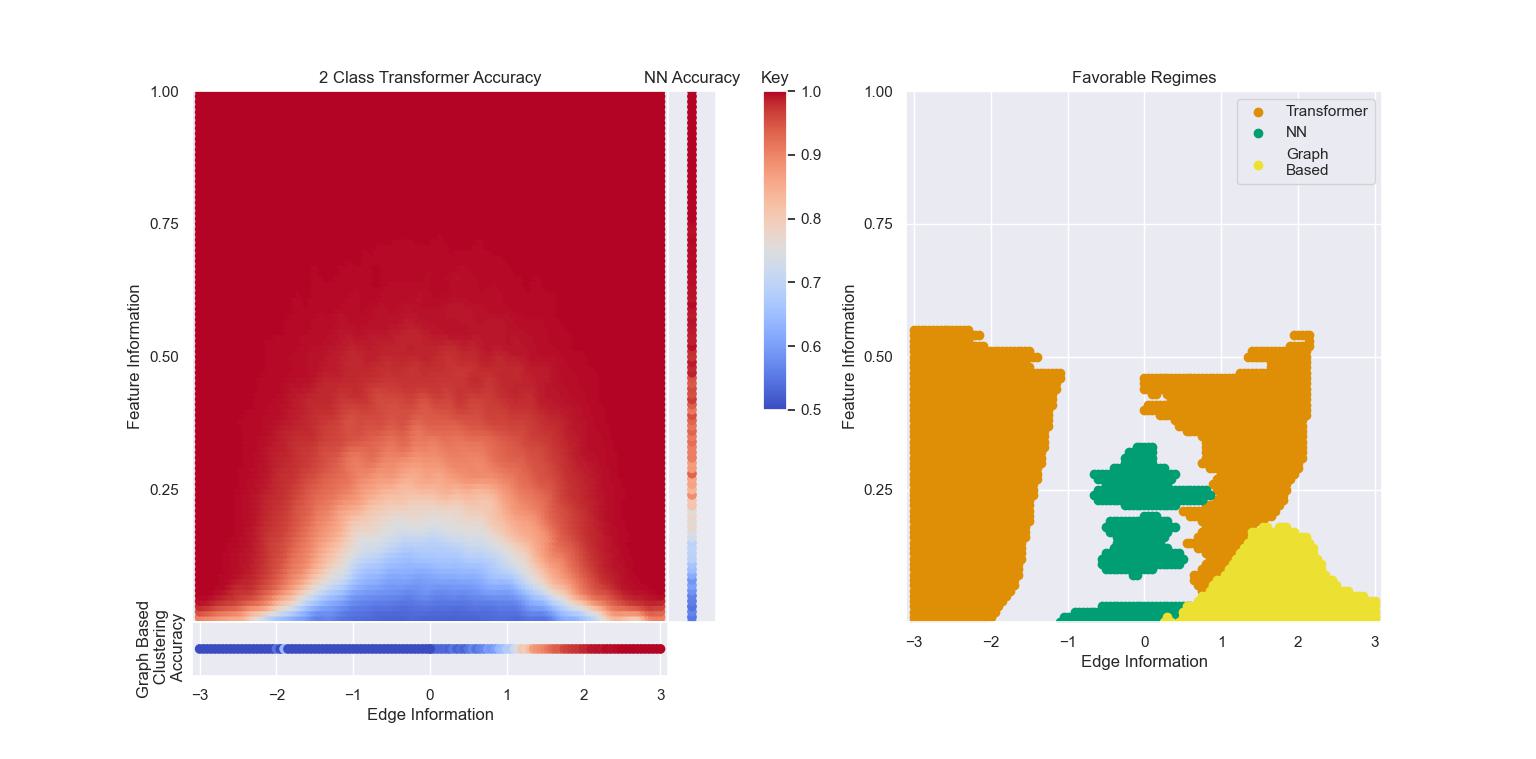}
    \caption{(Left) Transformer's performance on a two-class non-degree-corrected cSBM, with color gradients indicating accuracy levels. To the right and below, performance curves for the feedforward neural network (graph-blind) and graph-based (feature-blind) methodologies are displayed respectively. All reported scores represent the maximum of 10 independent trials, and a $5\times 5$ convolutional filter is applied to the visualization for enhanced clarity. (Right) A comparison of the top-performing model among the Graph Transformer, feedforward neural network, and graph-based clustering. White space indicates where one model was not consistently better than the others. 
    The Transformer predominantly excels when edge and feature information were moderately noisy. The graph based method is able to surpass the transformer if we have a combination of high feature noise and low edge noise. 
    }
    \label{fig:2class-transformer}
    \end{center}
\end{figure}

Our experiments with the Transformer architecture elucidate its robustness across a wide parameter space (see~\cref{fig:2class-transformer}). Remarkably, the Transformer consistently delivers superior performance across most scenarios, with exceptions only in cases where both the feature and edge information are heavily compromised by noise. An intriguing capability of the Transformer is its potential to achieve flawless accuracy even when presented with solely noisy edge information. This implies an innate adaptability within the Transformer to sift through the noise, selectively emphasizing pertinent features over less informative edges. Message-passing GNNs seem to struggle with this~\citep{bechlerspeicher2023graph} as seen in~\cref{fig:architecture_compare}.

The Transformer performs well on heterophilous graphs as well. Such proficiency makes the Transformer an excellent candidate for tasks demanding the assimilation of diverse or opposing sets of information. A marked limitation is observed in the Transformer's ability to process noisy feature scenarios, where spectral clustering performs better. The Transformer's somewhat dependent relationship with feature information, even when suboptimal, necessitates further investigation.

\subsection{Performance of GCN, GAT, SAGE, and Transformer architectures}
    \label{sec:architecture_compare}
    
\begin{figure}[htbp]
    \includegraphics[width=\textwidth]{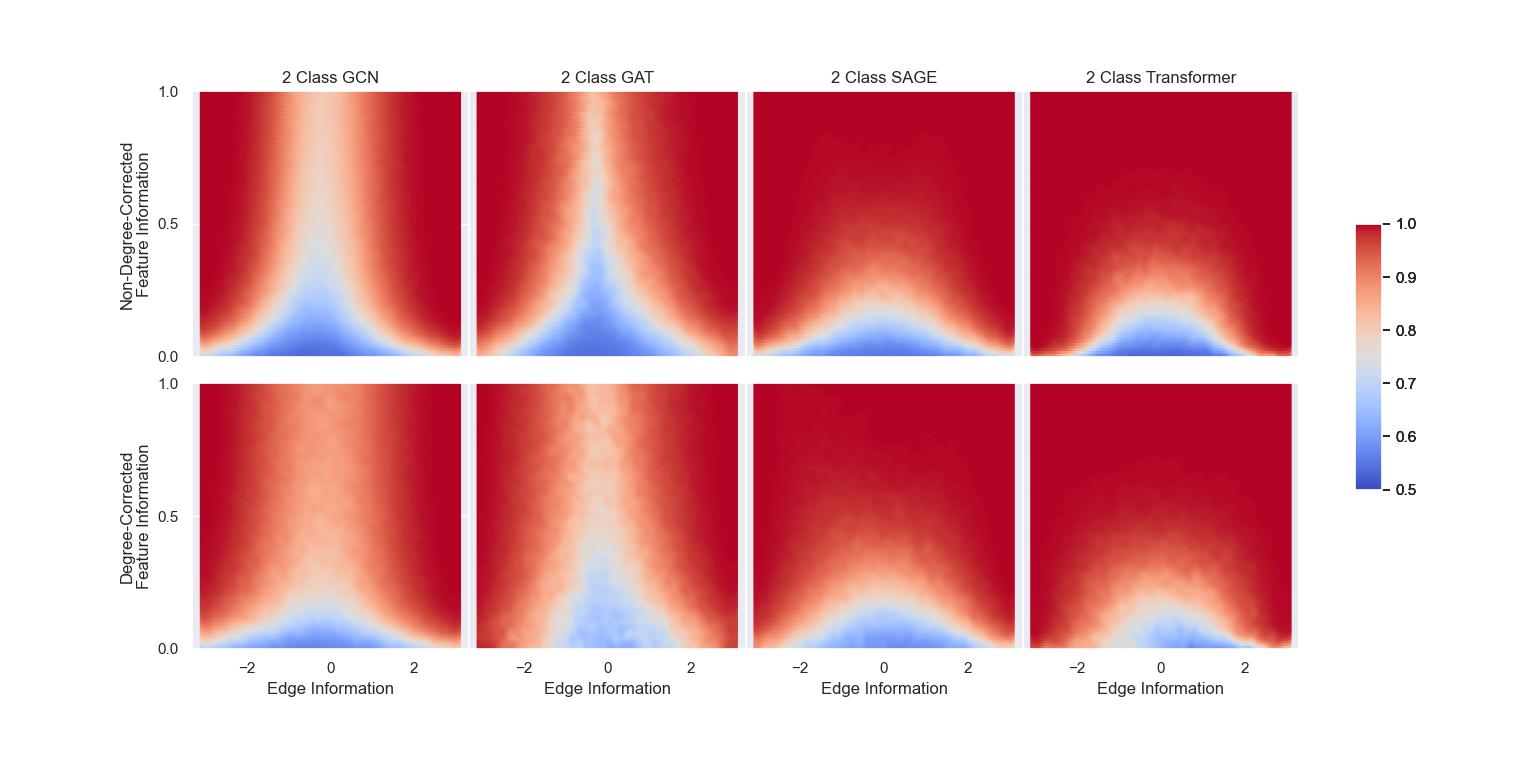}
    \caption{Comparison performance on non-degree-corrected and degree-corrected SBMs for GCN, GAT, SAGE and Transformer architectures. Notice the GCN and GAT consistently perform worse when the edge information is roughly zero, but the other two models are able to still achieve perfect accuracy given enough feature information. This could be due to SAGE and Transformer being able to learn a more global context for each node. 
    In this regime we see that almost all of the models did better on the heavy tailed graphs. GCN achieved higher accuracy on such graphs when the edges were just noise. The accuracy of the GAT improved as well in the regime of very noisy edges and features. All values are the best of 10 trials, with a $5\times 5$ convolutional filter applied for visual clarity.}
    \label{fig:architecture_compare}
\end{figure}

We now juxtapose the performances of four distinct architectures, particularly considering the influence of heavy-tailed degree distributions. Refer to \cref{fig:architecture_compare} for insights on the two-class scenario, while an exhaustive analysis is cataloged in \cref{appendix:means} and \cref{appendix:maxes}. Generally, both Graph-Transformer and SAGE stand out for their resistance to edge and feature noise, demonstrating their robustness in noisy regimes. In a two-class, non-degree-corrected cSBM setting, SAGE and Graph-Transformer consistently outperform the other two models, GAT and GCN. This is shown by their strong resistance to feature noise and their ability to classify accurately even without edge information. Such performance highlights SAGE's use of global information from random walks and graph embeddings, while the Transformer simply ignores the graph embedding.

Each architecture performs differently, as shown by their varying weak areas (seen as blue areas in \cref{fig:architecture_compare}) and how they compare to neural network and spectral clustering benchmarks (detailed in \cref{appendix:maxes}). The GAT and GCNs weak area is especially prominent with no edge information, showing it relies heavily on clear features. Interestingly, both Transformer and GAT perform better with degree correction, especially in heterophilous settings. 

\subsection{Degree-corrected SBMs}
    \label{sec:DC_GCN}

We found that all models performed better on scale-free graphs. We believe this occurs due to a filtering out of bad neighbors. Most nodes in the heavy-tailed data have relatively few neighbors, this allows for fewer confusing neighbors to contribute misleading information in the aggregation step than in the binomial degree distribution.
This is similar to ideas from~\citet{albert00}. 

The scale free graphs affected the models in different ways, for example the performance of SAGE only improved in the higher signal edge regimes (right and left sides of the \cref{fig:architecture_compare}). The performance of GAT increased dramatically in the case of very noisy edges and features. This is likely because GAT was already pruning bad edges, so perhaps degree correction gave it more information on what edges to prune. Interestingly, the attention based models, the Transformer and GAT, saw a stark increase in performance in the heterophilous clustering, suggesting that self-attention allows for a better interpretation of such graphs.

\section{Effect of higher-order structure in real world datasets}
    \label{sec:real_world}

\begin{figure}[htbp]
    \includegraphics[width=1\textwidth]{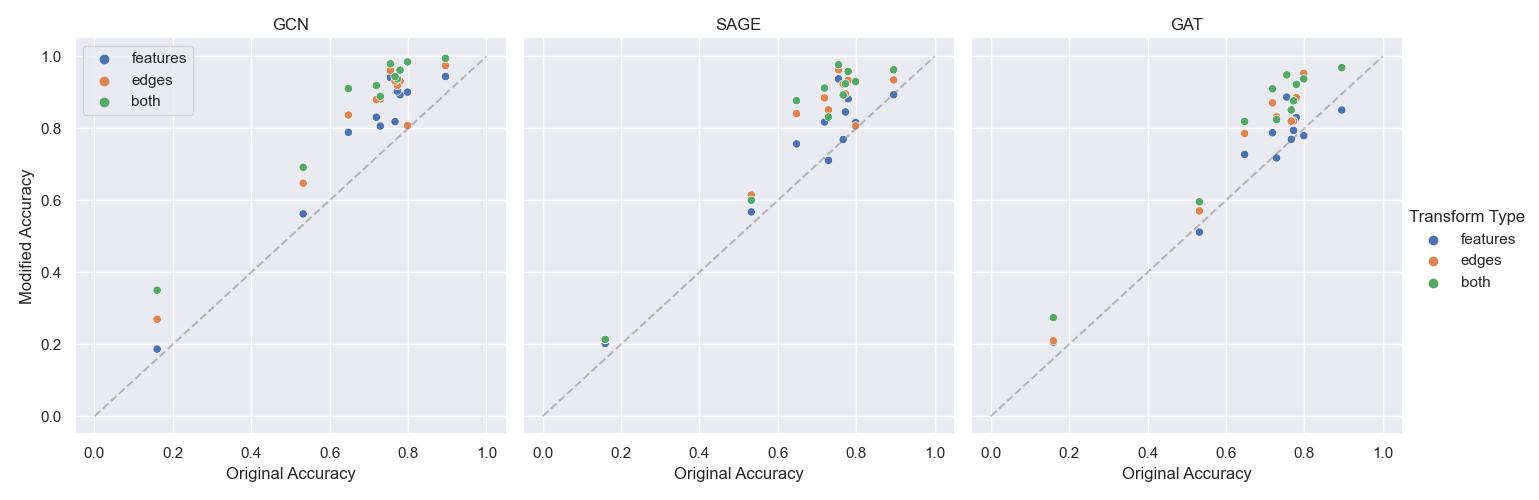}
    \caption{Comparison of model accuracies on real data compared to performance on matched synthetic data (with the same degree and clustering structure but without any higher-order structure). Each point represents one dataset. The accuracy tends to improve when we erase higher-order structure in the data. The datasets from left to right are: Flickr, DeezerEurope, Citeseer, LastFMAsia, DBLP, FacebookPagePage, Pubmed, GitHub, Cora, Amazon Computers, and Amazon Photos. The figure depicts cases where we transform only the edges, only the features, and both. The transformer was not run due to memory requirements.}
    \label{fig:real-compare}
\end{figure}

The experiments to be described in this section support the claim that higher-order structure, such as clustering or motifs, influence the performance of GNN architectures. In order to understand the impact of higher-order structure we devised a way to compare performance on real data (with natural higher-order structure) and closely matched synthetic data (with no higher-order structure).

We found that the models generally performed better on matched synthetic data than on real data, suggesting that the higher-order structure that was erased is an impediment to GNN learning (see~\cref{fig:real-compare}).

To make the synthetic data for each data set, we transformed the edge and feature data as if each dataset were already a degree-corrected cSBM. In particular, the edge data was randomized by rewiring every edge in a way that precisely preserved the empirical degree distribution and the number of edges both between and within all classes. This is a slight extension of the ideas in~\citet{fosdick_16}. In some experiments, the node features were also transformed by calculating the mean and standard deviation of the empirical features within each class, then sample independently from a Gaussian distribution to generate new node-level features that replace the empirical ones. Thus, the synthetic data  lacks nontrivial structure, except the degree distribution, intra/inter-class linkage frequency, and feature means and standard deviations match the corresponding empirical network.

We record the performance of a GCN on both the structured and unstructured version of 11 real world datasets. We use the following datasets obtained through pytorch geometric \citep{Fey/Lenssen/2019}: Flickr~\citep{zeng_graphsaint_2020}, DeezerEurope\citep{rozemberczki_characteristic_2020}, Citeseer~\citep{bojchevski_deep_2018}, LastFMAsia~\citep{rozemberczki_characteristic_2020}, DBLP~\citep{bojchevski_deep_2018}, FacebookPagePage~\citep{rozemberczki_multi-scale_2021}, Pubmed~\citep{bojchevski_deep_2018}, GitHub~\citep{rozemberczki_multi-scale_2021}, Cora~\citep{bojchevski_deep_2018}, Amazon Computers~\citep{shchur_pitfalls_2019},and Amazon Photos~\citep{shchur_pitfalls_2019}. We used the default train-test splits for training and evaluation. and restructured the data prior to training. Each dataset was trained using a 2 layer GCN over 200 epochs.

We see a positive impact on the accuracy of the GCN when removing the higher-order structure (see~\cref{fig:real-compare}) specifically with edge structure. When we scramble the edge data, we remove all statistically consistent structure within the graph except degree and community structure. Hence, structure like hierarchical clustering or locality bias are completely erased. The fact that the GNNs do better on this semi-randomized data suggests that they may perform optimally on SBM-like data, but are negatively impacted by the additional structure present in real data.

Uncovering why such structure can be detrimental to these GNNs is a significant opportunity for future work.

To further verify that we are not confusing higher-order structure with label noise, we verified these results on synthetic data with controlled structure. Such results indicate that GNNs perform worse on datasets with spatial structure, but are unaffected by local motifs such as triadic closure. Results on graphs with planted hierarchical structure were mixed but largely favored SBM data. A more detailed analysis can be found in \cref{appendix:higher_order}.

\section{Reproducibility Statement}
For further explanation of various proofs explored in \cref{sec:theory}, see \cref{sec:1layer_appendix} and \cref{appendix:2layer}. For code implementations of our studies in \cref{sec:architecture_compare} and \cref{sec:real_world}, see our GitHub or the supplementary material. For the exact implementation of \cref{sec:architecture_compare}, view the hyperparameters discussed in \cref{sec:experiment}. In regards to our findings in \cref{sec:real_world}, view \cref{appendix:higher_order} for a more in-depth explanation.
\section*{Acknowledgements}

We would like to thank the Office of Research Computing at BYU for providing the computing power necessary for this project. Likewise we would like to thank the ZiRui Su, Ryan Wood, Dustin Angerhoffer, Eli Child, and Carson Watkin for helpful discussions. We would also like to thank our the US National Science Foundation for support through award 2137511 and the US Army Research Lab for support through award W911NF1810244.

\bibliography{iclr2024_conference}
\bibliographystyle{iclr2024_conference}

\appendix

\section{Analysis of one-layer GCN}
\label{sec:1layer_appendix}

In this appendix, we prove the three parts of~\cref{thm:1layer}. 

\subsection{Distribution of linear embeddings}
\label{embedding_pf}

Analyzing the linear part of the model gives
\[ (AXW)_i = \sum_{j\in \mathcal{N}(i)} {X(j) W}.\]
From here, we split the sum into two parts corresponding to the two possible classes of neighbors:
\[\sum_{\substack{j \in \mathcal{N}(i) \\ j \text{ in class 1}}} {X(j) W} + 
    \sum_{\substack{j \in \mathcal{N}(i) \\ j \text{ in class 2}}} {X(j) W}\]
We then substitute the known expressions for $X(j)$:
\[\sum_{\substack{j \in \mathcal{N}(i) \\ j \text{ in class 1}}} {\left(\mu m + z_j\right)W} + 
    \sum_{\substack{j \in \mathcal{N}(i) \\ j \text{ in class 2}}} {\left(-\mu m + z_j\right)W}.\]
This becomes
\[(AXW)_i = \underbrace{\mu(\nin - \nout)mW}_\text{neighborhood signal} + \underbrace{\left(\sum_{j \in \mathcal{N}(i)} {z_j}\right) W}_\text{noise}. \]

\subsection{Nodewise accuracy, conditioned on the graph structure}
\label{sec:1layer_prob}

Assume $W\ne 0$. If $\nin=\nout=0$ then we have an isolated point. Since we are assuming no self loops and have no bias, these nodes do not affect the optimal parameters (in particular, the convolution outputs zero for these nodes). Thus we can assume that each node has at least one edge. We then compute,
\begin{align}
P(y[X](i) = 1  \mid \nin,\nout, v_i = 1) &= \int_0^\infty \frac{1}{\sqrt{2\pi W^T W(\nin + \nout)}} e^{-\frac{1}{2}\left(\frac{x-(\mu(\nin-\nout) mW)}{\sqrt{W^T W(\nin+\nout)}}\right)^2}\,dx
%\intertext{Distributing the square we have,}
% &= \int_{0}^{\infty}\frac{1}{\sqrt{W^T W(n_{\mathrm{in}} + n_{\mathrm{out}}) 2\pi}} e^{\frac{-(x-(\mu(n_{\mathrm{in}}-n_{\mathrm{out}}) mW))^{2}}{2W^T W(n_{\mathrm{in}}+n_{\mathrm{out}})}}\, dx
\end{align}
We now fix,
\begin{align}
u=\frac{x-(\mu(\nin-\nout) mW)}{\sqrt{2W^T W(\nin+\nout})} \ \text{with} \ du=\frac{dx}{\sqrt{2W^T W(\nin+\nout)}}.
\end{align}

 Notice $\sqrt{2W^T W(\nin+\nout)} > 0$ since each node has at least one edge and $W\ne 0$. We have, 
\[\frac{\sqrt{2W^T W(\nin+\nout)}}{\sqrt{2\pi W^T W(\nin+\nout)}}\int_{\frac{-(\mu(\nin-\nout) mW)}{\sqrt{2W^T W(\nin+\nout)}}}^{\infty}e^{-u^{2}} \ du \]
\[=\frac{1}{2} \int_{\frac{-(\mu(\nin-\nout) mW)}{\sqrt{2W^T W(\nin+\nout)}}}^{\infty}\frac{2e^{-u^{2}}}{\sqrt{\pi}} \ du= \frac{1}{2}\mathrm{erf}(u) \Big|_{\frac{-(\mu(\nin-\nout) mW)}{\sqrt{2W^T W(\nin+\nout)}}}^{\infty}\]
Observe that $\lim_{u \to \infty} \mathrm{erf}(u)=1$, so
\[\frac{1}{2}\left(\lim_{u \to \infty} \mathrm{erf}(u)-\mathrm{erf}\left(\frac{-(\mu(\nin-\nout) mW)}{\sqrt{2W^T W(\nin+\nout)}}\right) \right)\] \[= \frac{1}{2}\left(\mathrm{erf}\left(\frac{\mu(\nin-\nout) mW}{\sqrt{2W^T W(\nin+\nout)}}\right)+1\right).\]

Thus we have that \[P(y[X](i) = 1  \mid \nin,\nout, v_i = 1) = \frac{1}{2}\left(\mathrm{erf}\left(\frac{\mu(\nin-\nout) mW}{\sqrt{2W^T W(\nin+\nout)}}\right)+1\right),\]
as promised.
%By our lemma, we then achieve the following equality:
%\[P(D(v) = 1  \mid n_{\mathrm{in}},n_{{\mathrm{out}}}, C(v) = 1) = \frac{1}{2}\left(\mathrm{erf}\left(\frac{\mu(n_{\mathrm{in}}-n_{\mathrm{out}})}{\sqrt{2(n_{\mathrm{in}}+n_{\mathrm{out}})}}\cos(\theta)\right)+1\right).\]

\subsection{Maximizing accuracy}
\label{sec:1layer_max}

Given the symmetry of the linear model,
\[
P(y_i = v_i) = P(y_i = v_i | v_i = 1)
\]
Let $\Pin(\nin)$ be the probability of having $\nin$ homophilous edges, and $\Pout(\nout)$ be the probability of having $\nout$ heterophilous edges. Since $\Pin(\nin)$ and $\Pout(\nout)$ are independent we have,
\begin{align}
P(y_i = v_i | v_i = 1) &= P((AXW)_i>0 \mid v_i = 1)\\
&=\sum\limits_{\substack{\nin=0}}^\frac{N}{2} \sum\limits_{\substack{\nout=0}}^{\frac{N}{2}} P((AXW)_i>0\mid \nin, \nout)\Pin(\nin)\Pout(\nout).
\end{align}

%We now have that the expected total number of nodes assigned to the correct class with $N$ being the number of nodes is equal to: 
%\[N*P(x_{i}\text{ is correctly classified}).\]
Recall that $\theta$ is the angle between $W$ and $m$ and that consequently $\cos\theta = mW/\sqrt{W^TW}$. To find the maximizers, we now differentiate each term $P((AXW)_i>0\mid \nin, \nout)$ with respect to $\theta$ and set it equal to $0$.

\begin{align} 
\frac{d}{d\theta}\left(\frac{1}{2}\left(\mathrm{erf}\left(\frac{\mu(\nin-\nout)}{\sqrt{2(\nin+\nout)}}\cos(\theta)\right)+1\right)\right)&=0 \\
-\frac{\mu(\nin-\nout)}{\sqrt{2(\nin+\nout)}}\frac{1}{2}\sin(\theta)\left(\mathrm{erf}^{\prime}\left(\frac{\mu(\nin-\nout)}{\sqrt{2(\nin+\nout)}}\cos(\theta)\right)\right)&=0,\\
-\frac{\mu(\nin-\nout)}{\sqrt{2(\nin+\nout)}}\frac{1}{2}\sin(\theta)\frac{2}{\sqrt{\pi}}e^{-\left(\frac{\mu(\nin-\nout)}{\sqrt{2(\nin+\nout)}}\cos(\theta)\right)^{2}}&=0
\end{align}
where $\mathrm{erf}^{\prime}(x)=\frac{2}{\sqrt{\pi}}e^{-x^{2}}$. This is equal to $0$ exactly when $\theta=0$ and $\theta=\pi$. 

It turns out these are the only two critical values.
To demonstrate this we reintroduce our summations and re-index $\alpha=\nin+\nout$ and $\beta = \nin-\nout$. Define $U(\alpha)=\frac{N}{2}-\abs{\alpha- \frac{N}{2}}$:
\begin{align}
\label{eq:sum}
    \sum\limits_{\substack{\alpha=0}}^N \sum\limits_{\substack{\beta=-U(\alpha)}}^{U(\alpha)}-\frac{\mu\beta}{\sqrt{2\alpha\pi}}\sin(\theta)e^{-\frac{\mu^2\beta^2}{2\alpha}\cos^2(\theta)}\Pin(\nin(\alpha,\beta))\Pout(\nout(\alpha,\beta))\\
    =\frac{-\mu\sin(\theta)}{\sqrt{2\pi}}\sum\limits_{\substack{\alpha=0}}^N \sum\limits_{\substack{\beta=-U(\alpha)}}^{U(\alpha)}\frac{\beta}{\sqrt{\alpha}}e^{-\frac{\mu^2\beta^2}{2\alpha}\cos^2(\theta)}\Pin(\nin(\alpha,\beta))\Pout(\nout(\alpha,\beta))
\end{align}
by pairing off entries whose absolute value of beta are equal we have:
\begin{align}
    =\frac{-\mu\sin(\theta)}{\sqrt{2\pi}}\sum\limits_{\substack{\alpha=0}}^N \sum\limits_{\substack{\beta=1}}^{U(\alpha)}c(\alpha,\beta,\theta)\left(\Pin(\nin(\alpha,\beta))\Pout(\nout(\alpha,\beta))-\Pin(\nout(\alpha,\beta))\Pout(\nin(\alpha,\beta))\right)
\end{align}
with $c(\alpha,\beta,\theta)=\frac{\beta}{\sqrt{\alpha}}e^{-\frac{\mu^2\beta^2}{2\alpha}\cos^2(\theta)}$. From here note that in the case of homophily:
\begin{align}
    \Pin(\nin(\alpha,\beta))\Pout(\nout(\alpha,\beta))-\Pin(\nout(\alpha,\beta))\Pout(\nin(\alpha,\beta)) > 0
\end{align}
and heterophily:
\begin{align}
    \Pin(\nin(\alpha,\beta))\Pout(\nout(\alpha,\beta))-\Pin(\nout(\alpha,\beta))\Pout(\nin(\alpha,\beta)) < 0.
\end{align}
To see this, note that 
\[\Pin(\nin) \Pout(\nout) \propto \binom{\frac{N}{2}}{\nin} \binom{\frac{N}{2}}{\nout} \pin^\nin \pout^\nout.\]
Similarly, 
\[\Pin(\nout) \Pout(\nin) \propto \binom{\frac{N}{2}}{\nin} \binom{\frac{N}{2}}{\nout} \pin^\nout \pout^\nin.\]
Subtracting yields
\begin{align}
&\Pin(\nin) \Pout(\nout) - \Pin(\nout) \Pout(\nin) \\
&\propto \binom{\frac{N}{2}}{\nin} \binom{\frac{N}{2}}{\nout} \left( {\pin}^{\nin} \pout^\nout - \pin^\nout \pout^\nin\right) \\
&= \binom{\frac{N}{2}}{\nin} \binom{\frac{N}{2}}{\nout} \pin^\nout \pout^\nin \left( \left(\frac{\pin}{\pout}\right)^{\nin - \nout} - 1 \right).
\label{thing1}
\end{align}
Since $\nin \ge \nout$ (because $\beta \ge 0$), \cref{thing1} is positive in the heterophilous case and negative otherwise (unless $\nin = \nout$, of course). If $\Pin = \Pout$, we make no claims.

In any case, the first derivative is not equal to zero unless $\theta\in\{0,\pi\}$.

Notice if we include self-loops, the analysis case is very similar, with the caveat that there may rarely be another critical point in the very dense heterophilous case, due to the possibility of $\nin = \frac{N}{2}$.

Thus the critical points are $0$ and $\pi$.

% \begin{align}
%  -\frac{\mu(n_{\mathrm{in}}-n_{\mathrm{out}})}{\sqrt{2(n_{\mathrm{in}}+n_{\mathrm{out}})}}\frac{1}{2}\sin(0)\left(\frac{2}{\sqrt{\pi}}e^{-\left(\left(\frac{\mu(n_{\mathrm{in}}-n_{\mathrm{out}})}{\sqrt{2(n_{\mathrm{in}}+n_{\mathrm{out}})}}\right)\cos(0)\right)^{2}}\right) &=0 \\ 
% 0\cdot\left(\frac{2}{\sqrt{\pi}}e^{-\left(\frac{\mu(n_{\mathrm{in}}-n_{\mathrm{out}})}{\sqrt{2(n_{\mathrm{in}}+n_{\mathrm{out}})}}\right)^{2}}\right)&=0 \\
% 0&=0 \\
% -\frac{\mu(n_{\mathrm{in}}-n_{\mathrm{out}})}{\sqrt{2(n_{\mathrm{in}}+n_{\mathrm{out}})}}\frac{1}{2}\sin(\pi)\left(\frac{2}{\sqrt{\pi}}e^{-\left(\left(\frac{\mu(n_{\mathrm{in}}-n_{\mathrm{out}})}{\sqrt{2(n_{\mathrm{in}}+n_{\mathrm{out}})}}\right)\cos(\pi)\right)^{2}}\right) &=0 \\
% 0\cdot\left(\frac{2}{\sqrt{\pi}}e^{-\left(\frac{-\mu(n_{\mathrm{in}}-n_{\mathrm{out}})}{\sqrt{2(n_{\mathrm{in}}+n_{\mathrm{out}})}}\right)^{2}}\right)&=0 \\
% 0&=0
% \end{align}
We now take the second derivative with respect to $\theta$ to classify the critical points. For clarity we set $h=\frac{\mu(\nin-\nout)}{\sqrt{2(\nin+\nout)}}$.
Again, proceeding term by term gives
%\begin{align}
% P((AXW)_i>0 \mid v_i = 1)
%&=\sum\limits_{\substack{n_{\mathrm{in}}=1}}^\frac{N}{2} \sum\limits_{\substack{n_{\mathrm{out}}=0}}^{\frac{N}{2}} P(A\mathbf{x}_iW>0\mid n_{\mathrm{in}}, n_{\mathrm{out}}, v_i =1)P(n_{\mathrm{in}})P(n_{\mathrm{out}}) \\
%&=\sum\limits_{\substack{n_{\mathrm{in}}=1}}^\frac{N}{2} \sum\limits_{\substack{n_{\mathrm{out}}=0}}^{\frac{N}{2}}\frac{1}{2}\left(\mathrm{erf}\left(\frac{\mu(n_{\mathrm{in}}-n_{\mathrm{out}})}{\sqrt{2(n_{\mathrm{in}}+n_{\mathrm{out}})}}\cos(\theta)\right)+1\right)P(n_{\mathrm{in}})P(n_{\mathrm{out}}) 
%\end{align}
%Differentiating with respect to $\theta$ gives us
%\[\sum\limits_{\substack{n_{\mathrm{in}}=1}}^\frac{N}{2} \sum\limits_{\substack{n_{\mathrm{out}}=0}}^{\frac{N}{2}}-\frac{\mu(n_{\mathrm{in}}-n_{\mathrm{out}})}{\sqrt{2(n_{\mathrm{in}}+n_{\mathrm{out}})}}\frac{1}{2}\sin(\theta)\left(\frac{2}{\sqrt{\pi}}e^{-\left(\left(\frac{\mu(n_{\mathrm{in}}-n_{\mathrm{out}})}{\sqrt{2(n_{\mathrm{in}}+n_{\mathrm{out}})}}\right)\cos(\theta)\right)^{2}}\right)P(n_{\mathrm{in}})P(n_{\mathrm{out}})\]
%From our previous derivation we have that this is equal to 0 when $\theta=0,\pi.$ We now take the second derivative with respect to $\theta$ to classify the critical points.

\begin{align}
&\frac{-h}{\sqrt{\pi}}\frac{d}{d\theta}\left(\sin(\theta)e^{-h^2\cos^2\theta}\right) \\
&=\frac{-2h^3}{\sqrt{\pi}}\sin(\theta)^{2}\cos(\theta) e^{-h^2\cos^2\theta}
-\frac{h}{\sqrt{\pi}}\cos(\theta)e^{-h^2\cos^2\theta}.
\end{align}
Since $\sin(\theta) = 0$ at both critical points and $\cos(\theta)=\pm 1$, this simplifies to
\begin{align}
=-\frac{h}{\sqrt{\pi}}\cos(\theta)e^{-h^2}.
\end{align}
We now reintroduce the summations and reindex. Once again fixing $\alpha=\nin+\nout$ and $\beta=\nin-\nout$. Let $U(\alpha)$ be defined as above. The second derivative is then
\begin{align}
&\frac{-\mu}{\sqrt{2\pi}}\cos(\theta)\sum\limits_{\substack{\alpha}=0}^N\sum\limits_{\substack{\beta=-U(\alpha)}}^{U(\alpha)}\frac{\beta}{\sqrt{\alpha}} e^{-\frac{\mu^{2}\beta^{2}}{2\alpha}} P(\nin(\alpha,\beta))P(\nout(\alpha,\beta)) \\ 
&=\frac{-\mu}{\sqrt{2\pi}}\cos(\theta)\sum\limits_{\substack{\alpha}=0}^N\sum\limits_{\substack{\beta=1}}^{U(\alpha)}
    \frac{\beta}{\sqrt{\alpha}} e^{-\frac{\mu^{2}\beta^{2}}{2\alpha}} \left(P_{\mathrm{in}}(\nin(\alpha,\beta)P_{\mathrm{out}}(\nout(\alpha,\beta)) - P_{\mathrm{in}}(\nout(\alpha,\beta))P_{\mathrm{out}}(\nin(\alpha,\beta))\right)
\end{align}
Similar to the analysis with the first derivative, the second term in the innermost sum is always less than the first (assuming homophily here), we that the second derivative must be positive at $\pi$ and negative at $0$, as expected. In the heterophilous case, the opposite sign rules apply.

Thus in the homophilous case the maximal accuracy is obtained when $\theta=0$ or our weight matrix is pointing in the same direction as our average feature vector. The minimal accuracy is obtained with $\theta=\pi$. For heterophily reversed rules apply.
%We now use a matching technique to take advantage of homophily or heterphily to determine the sign.
%For example, under homophily, for all $1 \le \alpha \le N,$ for all $-\frac{N}{2} + 1 \le \beta^- < 0,$ there exists $\beta^+ = -\beta^-$ such that 
%\begin{align} 
%\frac{\beta^-}{\sqrt \alpha} e^\frac{-\mu^2\beta^{-^2}}{2\alpha}  = \frac{-\beta^+}{\sqrt \alpha}e^\frac{-\mu^2\beta^{+^2}}{2\alpha} 
%\end{align}
%and
%\begin{align}
%P\left(\frac{\alpha+\beta^-}{2}\right)P\left(\frac{\alpha-\beta^-}{2}\right) \le P\left(\frac{\alpha+\beta^+}{2}\right)P\left(\frac{\alpha-\beta^+}{2}\right).
%\end{align}
%Therefore the value over the entire sum becomes negative (excluding the value for $\cos(\theta)),$ and we see that, under homophily, $\theta = 0$ is a maximizer and $\theta = \pi$ is a minimizer for the network.
%
%
%The converse is also true. Namely, under heterophily, for all $1 \le \alpha \le N,$ for all $0 < \beta^+ \le \frac{N}{2},$ there exists $\beta^- = -\beta^+$ such that 
%
%\begin{align}e^\frac{-\mu^2\beta^{+^2}}{2\alpha} \frac{\beta^+}{\sqrt \alpha} = -e^\frac{-\mu^2\beta^{-^2}}{2\alpha} \frac{\beta^-}{\sqrt \alpha}
%\end{align}
%and
%\begin{align}
%P\left(\frac{\alpha+\beta^+}{2}\right)P\left(\frac{\alpha-\beta^+}{2}\right) \le P\left(\frac{\alpha+\beta^-}{2}\right)P\left(\frac{\alpha-\beta^-}{2}\right).
%\end{align}
%Therefore the value over the entire sum becomes positive (excluding the value for $\cos(\theta)),$ and we see that, under heterophily, $\theta = \pi$ is a maximizer and $\theta = 0$ is a minimizer for the network.

\section{Analysis of two-layer GNNs}
\label{appendix:2layer}

\subsection{Proof that linear models are optimal in certain cases}

Let $\Omega$ be a 2-class attributed random graph model. For any $x = (G, i, v, X) \in \Omega$, we defined earlier the negation
\[
-x = (G, i, -v, -X)
\]
and the reflection
\[
R_S(x) = (G, i, v, R_S \circ X)
\]
where $S$ is some subspace of $\R^m$. We also define $P_\perp = I - P_S$ or the projection onto the subspace orthogonal to $S$. 
% TREVOR: did we ever say that the features live in Euclidean space? Because we are now using an inner product for the projection (or something). I wonder if this works over any inner product space.

\begin{lemma}
    For any model $y$ on $\Omega$, the cost function is given by:
    \[
    C(y) = \mathbb{E}_{x \sim \Omega} \left[ \log\left(1 + e^{-v(x)y(x)}\right)\right]
    \]
\end{lemma}
\begin{proof}
    By definition,
    \[C(y) = 
    \mathbb{E}_{x \sim \Omega} [-\log P(v(x) : y(x))]
    \]
    where
    \[
    P(v(x) : y(x)) = \left\{\begin{array}{lr}
\sigma_s(y(x)) & v(x) = 1\\
1- \sigma_s(y(x)) & v(x) = -1
\end{array} \right.
    \]
    and $\sigma_s(z) = (1 + e^{-z})^{-1}$. Famously the sigmoid function satisfies $1 - \sigma_s(z) = \sigma_s(-z)$. We can then re-write the probability as
    \[
    P(v(x) : y(x)) = \sigma_s(v(x)y(x))
    \]
    Using the additionally identity $-\log \sigma_s(z) = \log(1 + e^{-z})$ we obtain,
    \[
    C(y) = \mathbb{E}_{x \sim \Omega} \left[-\log \sigma_s(v(x)y(x))\right] = \mathbb{E}_{x \sim \Omega} \left[ \log\left(1 + e^{-v(x)y(x)}\right)\right]
    \]
\end{proof}

\begin{lemma}
    The function $f(x) = \log(1 + e^{-x})$ is convex.
\end{lemma}
\begin{proof}
    It suffices to take the second derivative:
    \[
    f''(x) = \frac{e^x}{(1+e^x)^2} > 0.
    \]
\end{proof}

\begin{lemma}
    Let $y$ be any model on $\Omega$. If $\Omega$ is class-symmetric about the origin, then following inequality holds:
    \[
    C(y) \geq \mathbb{E}_{x \sim \Omega} \left[ \log\left(1 + e^{-v(x)\frac{y(x) - y(-x)}{2}}\right)\right]
    \]
    If $\Omega$ is symmetric about the subspace $S$, then
    \[
    C(y) \geq \mathbb{E}_{x \sim \Omega} \left[ \log\left(1 + e^{-v(x)\frac{y(x) + y\left(R_S(x)\right)}{2}}\right)\right]
    \]
\end{lemma}
\begin{proof}
    First let $\Omega$ be class-symmetric about the origin. By the above lemma,
    \[
    C(y) = \mathbb{E}_{x \sim \Omega} \left[ \log\left(1 + e^{-v(x)y(x)}\right)\right]
    \]
    Since $P(F) = P(-F)$ for all $F \subset \Omega$, we can make a change of variables $x \mapsto -x$ to obtain
    \[
        C(y) = \mathbb{E}_{x \sim \Omega} \left[ \log\left(1 + e^{-v(-x)y(-x)}\right)\right] = \mathbb{E}_{x \sim \Omega} \left[ \log\left(1 + e^{v(x)y(-x)}\right)\right]
    \]
    We may therefore add the two expressions and divide by 2 to arrive at,
    \[
    C(y) = \frac{1}{2}\mathbb{E}_{x \sim \Omega} \left[\log\left(1 + e^{-v(x)y(x)}\right) + \log\left(1 + e^{v(x)y(-x)}\right)\right].\]
    By Jensen's inequality for convex functions,
    \[
    \frac{1}{2}\left[f(z_1) + f(z_2)\right] \geq f\left(\frac{z_1+z_2}{2}\right)
    \]
    for $f(z) = \log(1 + e^{-z})$.
    Applying this to $C[y]$, we obtain
    \begin{align}
    C(y) &\geq \log\left(1 + e^{\frac{1}{2}(-v(x)y(x) + v(x)y(-x))}\right)\\
    &= \log\left(1 + e^{-v(x)\frac{y(x) - y(-x)}{2}}\right)
    \end{align}

    If $\Omega$ is symmetric about the subspace $S$, then the same reasoning yields,
    \[
    C(y) \geq \log\left(1 + e^{-v(x)\frac{y(x) + y(R_S(x))}{2}}\right).
    \]
    Note that there is a sign difference from the previous expression, as negation flips the classes while reflection does not.
\end{proof}

Recall that if
\[
y(x) = (\phi'_G \circ \sigma \circ \phi_G)[X](i)
\]
then
\[
L[y](x) = \frac{1}{2}(\phi'_G \circ \phi_G)[X](i)
\]
and
\[
P_S[L[y]](x) = \frac{1}{2}(\phi'_G \circ \phi_G \circ P_S)[X](i)
\]
\begin{lemma}
    Let $y$ be any generalized 2-layer GCN without bias on $\Omega$. Then for any $x \in \Omega$,
    \[
    \frac{y(x)- y(-x)}{2} = L[y](x)
    \]
    and for any subspace $S$ of $\R^m$,
    \[
    \frac{L[y](x) + L[y](R_S(x))}{2} = R_S[L[y]](x)
    \]
\end{lemma}
\begin{proof}
    \begin{align}
        y(x) - y(-x) &= (\phi'_G \circ \sigma \circ \phi_G)[X](i) - (\phi'_G \circ \sigma \circ \phi_G)[-X](i)\\
        &= \phi'_G\big(\sigma(\phi_G(X)) - \sigma(\phi_G(-X)) \big)(i) \quad (\text{by linearity of } \phi_G')\\
        &= \phi'_G\big(\sigma(\phi_G(X)) - \sigma(-\phi_G(X)) \big)(i) \quad (\text{by linearity of } \phi_G)\\
        &= \phi_G'(\phi_G(X))(i) \quad (\text{as } \sigma(z) - \sigma(-z) = z)\\
        &= 2L[y](x)
    \end{align}
    which after dividing by $2$ proves the first expression. Next,
    \begin{align}
    &2(L[y](x) + L[y](R_S(x)))\\
    &= (\phi'_G \circ \phi_G)[X](i) + (\phi'_G \circ \phi_G)[R_S(X)](i)\\
    &= (\phi'_G \circ \phi_G)[P_S(X) + P_\perp(X)](i) + (\phi'_G \circ \phi_G)[P_S(X) - P_\perp(X)](i)\\
    &= 2(\phi'_G \circ \phi_G)[P_S(X)](i) \quad (\text{by linearity of }\phi'_G \text{ and } \phi_G)\\
    &= 4P_S[L[y]](x)
    \end{align}
    which after dividing by $4$ proves the second expression.
\end{proof}

\begin{theorem}
    Let $\Omega$ be a 2-class attributed random graph model and let $y$ be any two-layer GCN without bias on $\Omega$. If $\Omega$ is class-symmetric about the origin then,
    \[
    C(L[y]) \leq C[y].
    \]
    Furthermore, if $\Omega$ is symmetric about $S$ then,
    \[
    C(P_S[L[y]]) \leq C(L[y])
    \]
\end{theorem}
\begin{proof}
    If $\Omega$ is class-symmetric about the origin then,
    \begin{align}
        C[y] &\geq \mathbb{E}_{x \sim \Omega} \left[ \log\left(1 + e^{-v(x)\frac{y(x) - y(-x)}{2}}\right)\right]\\
        &=  \mathbb{E}_{x \sim \Omega} \left[ \log\left(1 + e^{-v(x)L[y](x)}\right)\right]\\
        &= C[L[y]]
    \end{align}
    Similarly, if $\Omega$ is symmetric about $S$ then the above lemmas applied to $L[y]$ yield,
    \begin{align}
        C(L[y]) &\geq \mathbb{E}_{x \sim \Omega} \left[ \log\left(1 + e^{-v(x)\frac{L[y](x) + L[y](R_S(x))}{2}}\right)\right]\\
        &= \mathbb{E}_{x \sim \Omega} \left[ \log\left(1 + e^{-v(x)P_S[L[y]](x)}\right)\right]\\
        &= C(P_S[L[y]]).
    \end{align}
\end{proof}

\subsection{Accuracy analysis in the optimal-parameter case}

In light of the preceding theorem, we now study the accuracy of the optimal linear, two-layer GCN, which we are able to compute in integrals. Let
\[
y(x) = K\sum_{j \in \mathcal{N}(x)}\sum_{k \in \mathcal{N}(j)} X(k) \cdot m
\]
over a cSBM with expected average node degree $d$ and edge information parameter $\lambda$, where $K$ is a constant, direction $m \in \R^{m_\mathrm{feat}}$, and features are given by
\[
X(i) = v_i \mu m + z_i
\]
where $v_i \in \{\pm 1\}$ is the class and $z_i$ is the Gaussian error with mean 0 and variance $\sigma^2 I$. In this case, $X(i) \cdot m$ is given by
\[
X(i) \cdot m = v_i \mu + z_i \cdot m = v_i + b_i
\]
where $b_i = z_i \cdot m$ is Gaussian with mean $0$ and variance $\sigma^2$.

In our analysis, self-loops will be added. Furthermore, $\din$ and $\dout$ will denote,
\[
\din = \frac{d + \lambda \sqrt{d}}{2}, \quad \dout = \frac{d - \lambda \sqrt{d}}{2}.
\]

In the large node limit, the number of neighbors of a node $i$ having the same class, $\nin$, is distributed according to a Poisson distribution with mean $\din$. Similarly, the number of neighbors having the opposite class, $\nout$, is distributed according to a Poisson distribution with mean $\dout$.

The number of same class neighbors of the same-class neighbors of $i$, denoted $\ninin$ is given by a Poisson distribution conditional on $\nin$ with mean $\din \nin$. Similarly the number of opposite class neighbors of the same class neighbors of $i$, denoted $\ninout$ is given by a Poisson distribution conditional on $\nin$ with mean $\dout \nin$. We define $\noutin$ and $\noutout$ similarly.

Let $\nintwo$ and $\nouttwo$ denote $\ninin+\noutout$ and $\ninout+\noutin$ respectively. 
Intuitively, $\nintwo$ and $\nouttwo$ denote the number of same class and opposite class nodes distance two from node $i$. 
By independence, $\nintwo$ and $\nouttwo$ are given by a Poisson distribution conditional on $\nin$ and $\nout$ with means 
$\din\nin + \dout\nout$ and 
$\dout\nin+\din\nout$, respectively. 
Then, if we let $p(k,\lambda) = \frac{\lambda^k e^{-\lambda}}{k!}$ be the pmf of the Poisson distribution, the probability of $\nin$, $\nout$, $\nintwo$, and $\nouttwo$ occurring can be factored as
\begin{align}
&P(\nin, \nout, \nintwo, \nouttwo)\\
&= p(\nin, \din)\cdot p(\nout, \dout)
\cdot p(\nintwo, \din\nin + \dout\nout)\cdot p(\nouttwo, \dout\nin+\din\nout).
\end{align}

Given $\nin$, $\nout$, $\nintwo$, and $\nouttwo$, the model $y(x)$ will have mean
\[
\mu K \sum_{j \in \mathcal{N}(x)}\sum_{k \in \mathcal{N}(j)} v_k
\]
as the error terms have mean 0. Taking self-loops into account, there are $(\nin + \nout + 1)$ 2-walks to the central node, two 2-walks to each of the neighbors, and one 2-walk to each of the nodes at distance 2. Recall the mean may be calculated linearly while variance satisfies
\[
\Var\left(\sum_i a_iX_i \right) = \sum_i a_i^2 \Var(X_i)
\]
where the $\{X_i\}_i$ are independent distributions. The conditional mean is therefore given by
\[
K\mu v(x)\bigg((\nin + \nout + 1) + 2(\nin - \nout) + \nintwo - \nouttwo \bigg)
\]
and variance
\[
K^2 \sigma^2 \bigg((\nin+\nout+1)^2 + 4(\nin+\nout) + (\nintwo + \nouttwo) \bigg).
\]
When the graph structure is fixed, the model outputs will be be Gaussian-distributed (as it is a sum of Gaussian clouds), and its accuracy is the probability that its sign matches $v(x)$. By symmetry, we may assume $v(x) = 1$. If $\Phi$ is the cdf of the standard distribution, then this accuracy is given by $\Phi$ applied to the mean divided by the standard deviation. The accuracy is then,
\begin{align}
&\Phi \left(\frac{K\mu(1 + 3\nin - \nout + \nintwo - \nouttwo)}{|K|\sigma\sqrt{(\nin + \nout + 1)^2 + 4(\nin + \nout) + (\nintwo + \nouttwo)}}\right)\\
&= \Phi\bigg(\psi\bigg(\sgn(K)\tfrac{\mu}{\sigma}, \nin, \nout, \nintwo, \nouttwo\bigg)\bigg)
\end{align}
where
\[
\psi(c, \nin, \nout, \nintwo, \nouttwo) = c\frac{1 + 3\nin - \nout + \nintwo - \nouttwo}{\sqrt{(\nin + \nout + 1)^2 + 4(\nin + \nout) + (\nintwo + \nouttwo)}}.
\]
The total accuracy is then given by
\[
\sum_{\nin, \nout, \nintwo, \nouttwo = 0}^\infty P(\nin, \nout, \nintwo, \nouttwo) \Phi\bigg(\psi\bigg(\sgn(K)\tfrac{\mu}{\sigma}, \nin, \nout, \nintwo, \nouttwo\bigg)\bigg).
\]

\section{Complete set of accuracy maps}
\subsection{Means}
    \label{appendix:means}
The comprehensive results for mean values of our experiments are found below. Additional experiments using other architectures or wider bounds may be conducted using our code in GitHub 
%https://github.com/brownthesr/Synthetic-Graphs.git ### ADD IN AFTER PEER REVIEW

\begin{figure}[htbp]
    \centering
    \textbf{Binomial Degree Distribution}\par\medskip
    \includegraphics[width=\textwidth]{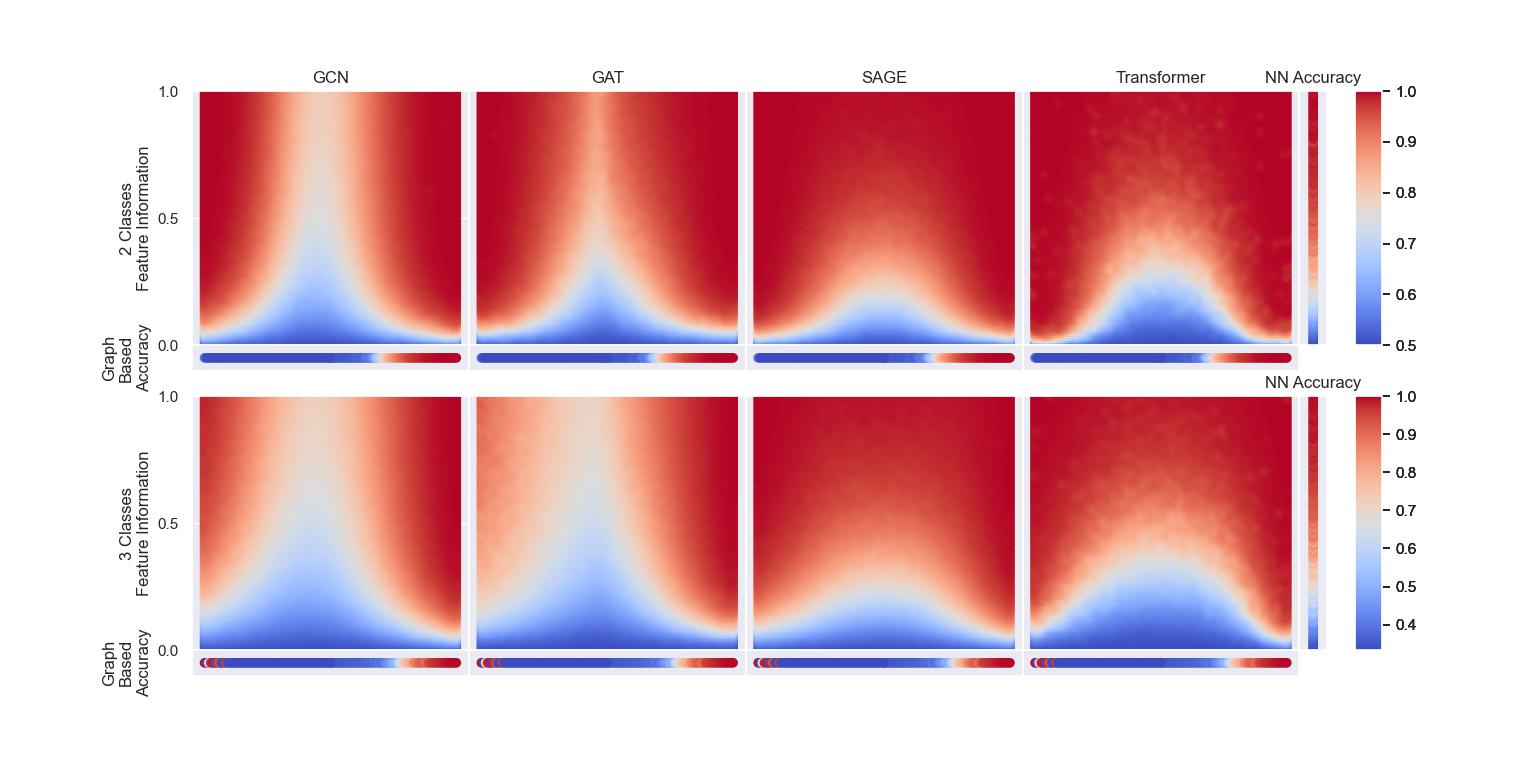}
    \includegraphics[width=\textwidth]{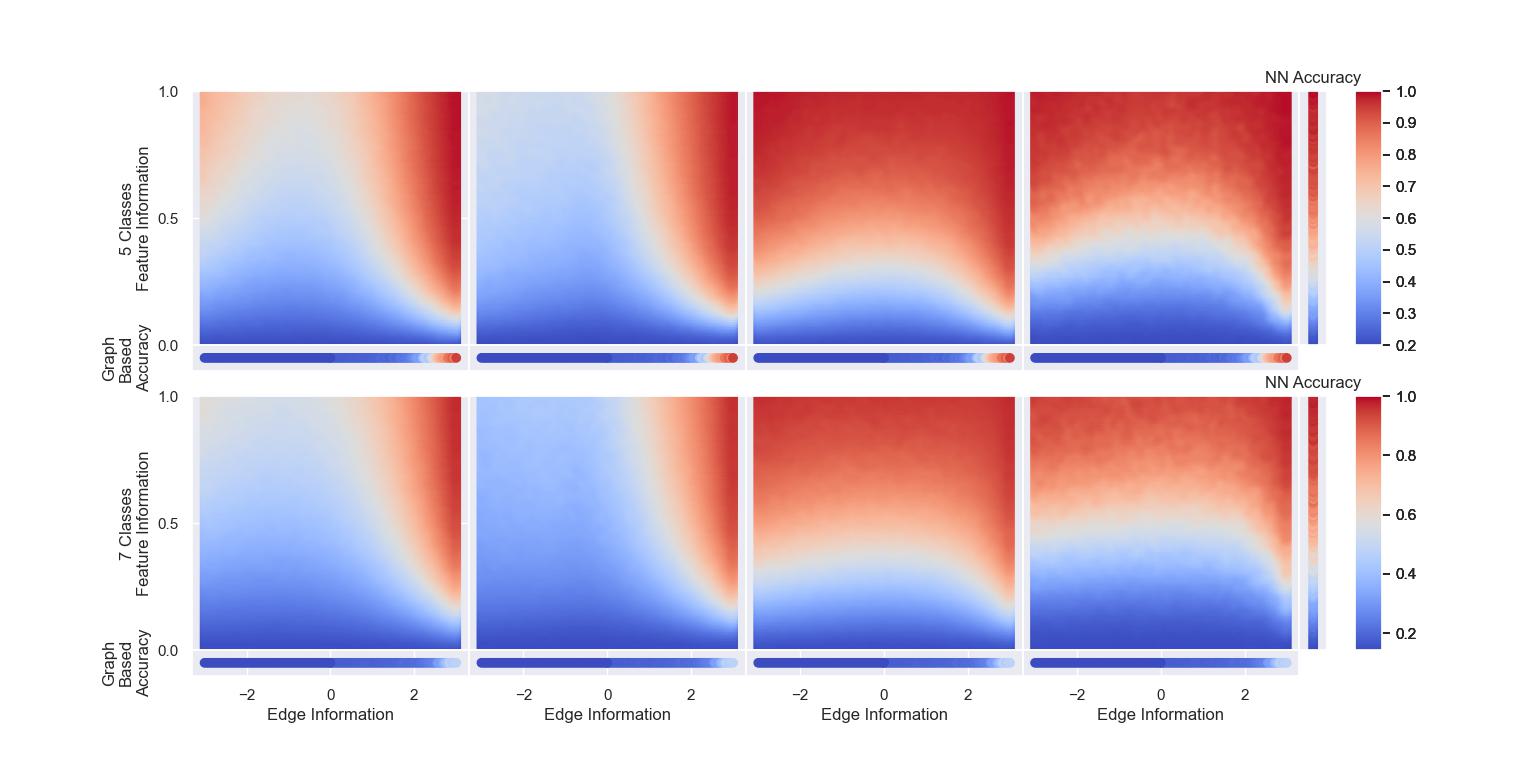}
    \caption{We compare accuracies over the distributed graphs across varying class sizes. We also depict the accuracy curves of a regular feedforward neural network and that of spectral clustering on the same datasets to the right and bottom of each plot respectively. Using these plots we compare how well each architecture performs on an increased number of classes. Additionally we can view how performance changes across different architectures.}
    \label{fig:mean_panels}
\end{figure}
When considering the Binomial SBM, we see that SAGE performed the best of any GNN architecture across any class size. As we increase the number of classes more information is needed for any architecture to classify correctly. Additionally, the increase in class size more adversely affects the heterophilous regime than the homophilous regime. By comparing the figures in ~\cref{fig:mean_panels} and~\cref{fig:dc_mean_panels} we can observe how each model is affected by degree correction across any number of classes.
\begin{figure}[htbp]
    \centering
    \textbf{Degree-Corrected Degree Distribution}\par\medskip

    \includegraphics[width=\textwidth]{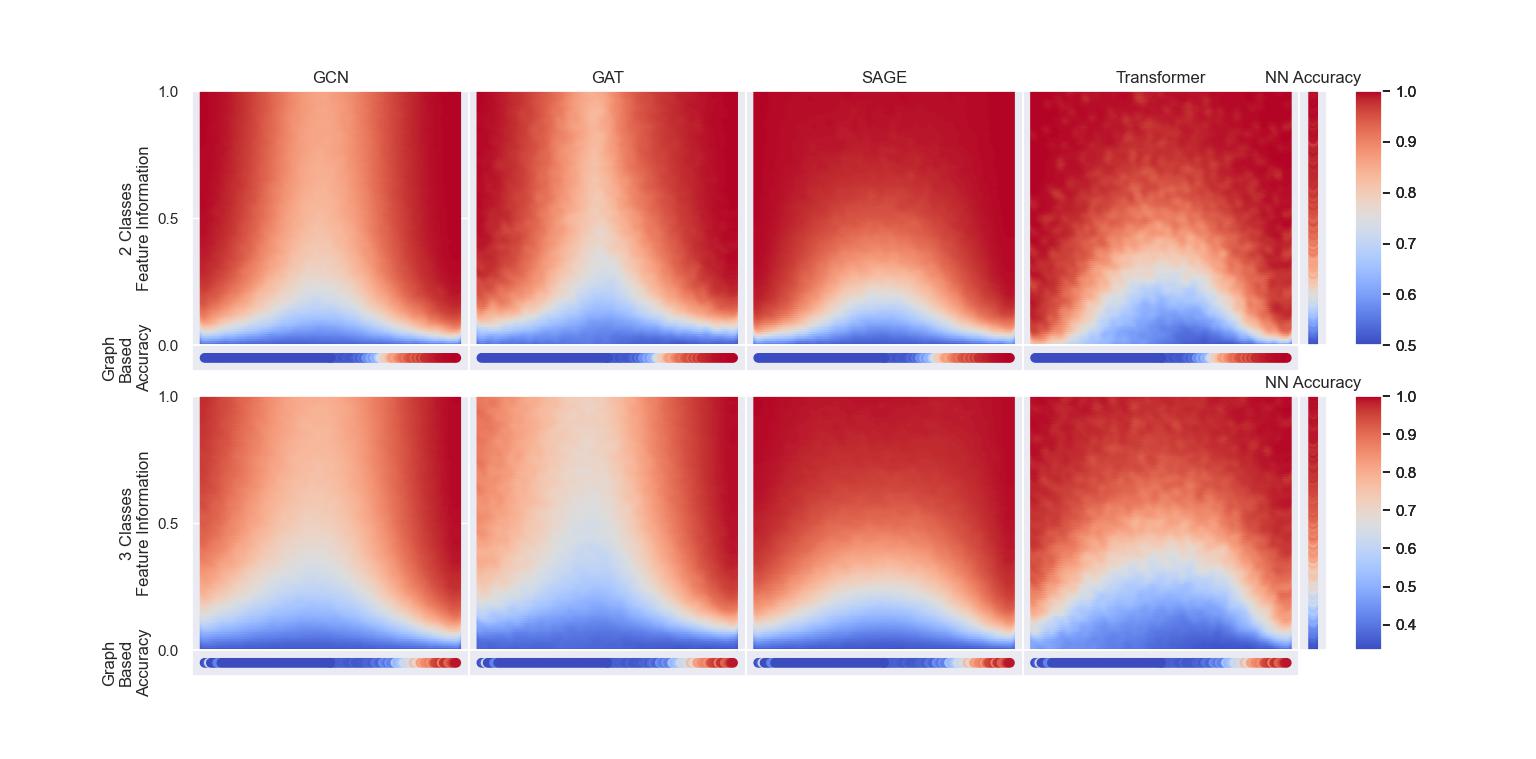}
    
    \includegraphics[width=\textwidth]{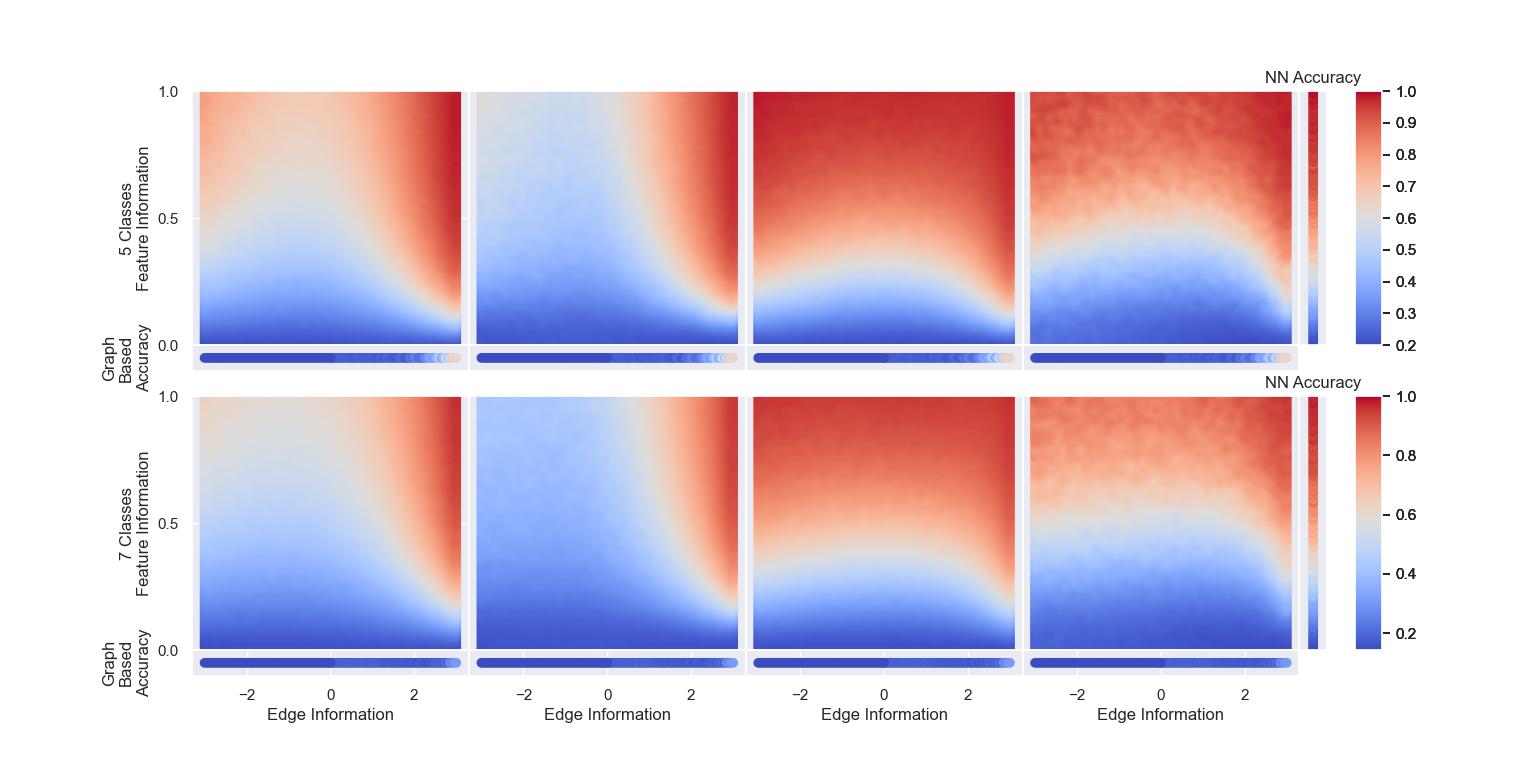}
    \caption{We compare accuracies over degree-corrected graphs with varying class sizes. The GCN did better on degree corrected graphs across any number of classes. This can be observed by viewing how the blue region in the top figures shrinks in the degree-corrected case. The performance of the Transformer improved in degree-corrected cases for class numbers of two and three, yet it decreased performance for class numbers of five and seven. The performance of SAGE and GAT were mostly unaffected by the degree correction.}
    \label{fig:dc_mean_panels}
\end{figure}

In~\cref{fig:mean_regions} and~\cref{fig:dc_mean_regions} we view the various regimes across which each architecture outperforms the others. As we increase the class sizes, the favorable regime for the neural network increases in size, showing that in many cases it is simply better to ignore edges and utilize solely the feature information. However, it should be noted that in most of the cases, there is always a regime where the GNN architecture outperforms both of the baselines.

\begin{figure}[htbp]
    \centering
    \textbf{Binomial Degree Distribution}\par\medskip

    \includegraphics[width=\textwidth]{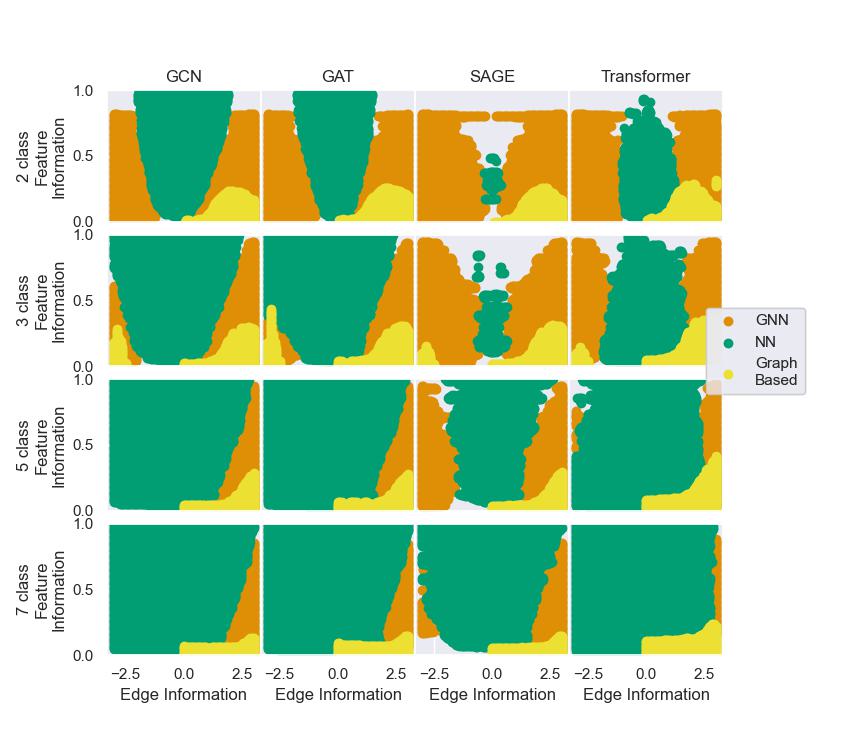}
    \caption{Regions depicting where each architecture outperforms the others across the SBM graphs. Here we can compare how varying parts of the data effects the shape and sizes of the favorable regimes}
    \label{fig:mean_regions}
\end{figure}
\begin{figure}[htbp]
    \centering
    \textbf{Degree-Corrected Degree Distribution}\par\medskip

    \includegraphics[width=\textwidth]{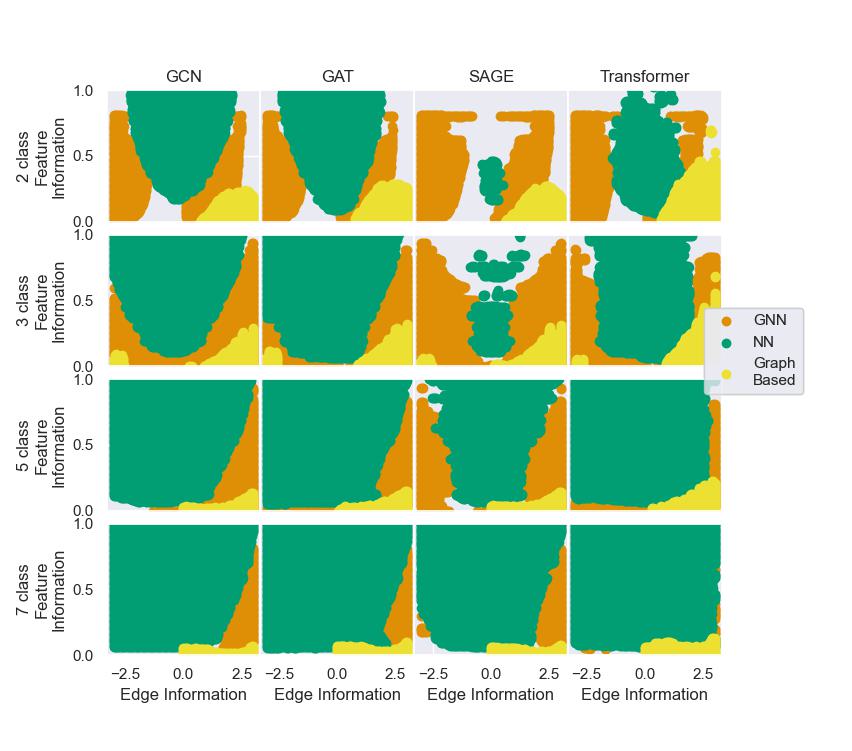}
    \caption{Regions where each architecture outperforms the others across degree-corrected graphs.}
    \label{fig:dc_mean_regions}
\end{figure}
\newpage
\subsection{Maxes}
\label{appendix:maxes}
The comprehensive results for max values of our experiments are found below. Additional experiments using other architectures or wider bounds may be conducted using our code in our GitHub.
%github https://github.com/brownthesr/Synthetic-Graphs.git.

\begin{figure}[htbp]
    \label{fig:max_panels}
    \centering
    \textbf{Binomial Degree Distribution}\par\medskip
    \includegraphics[width=\textwidth]{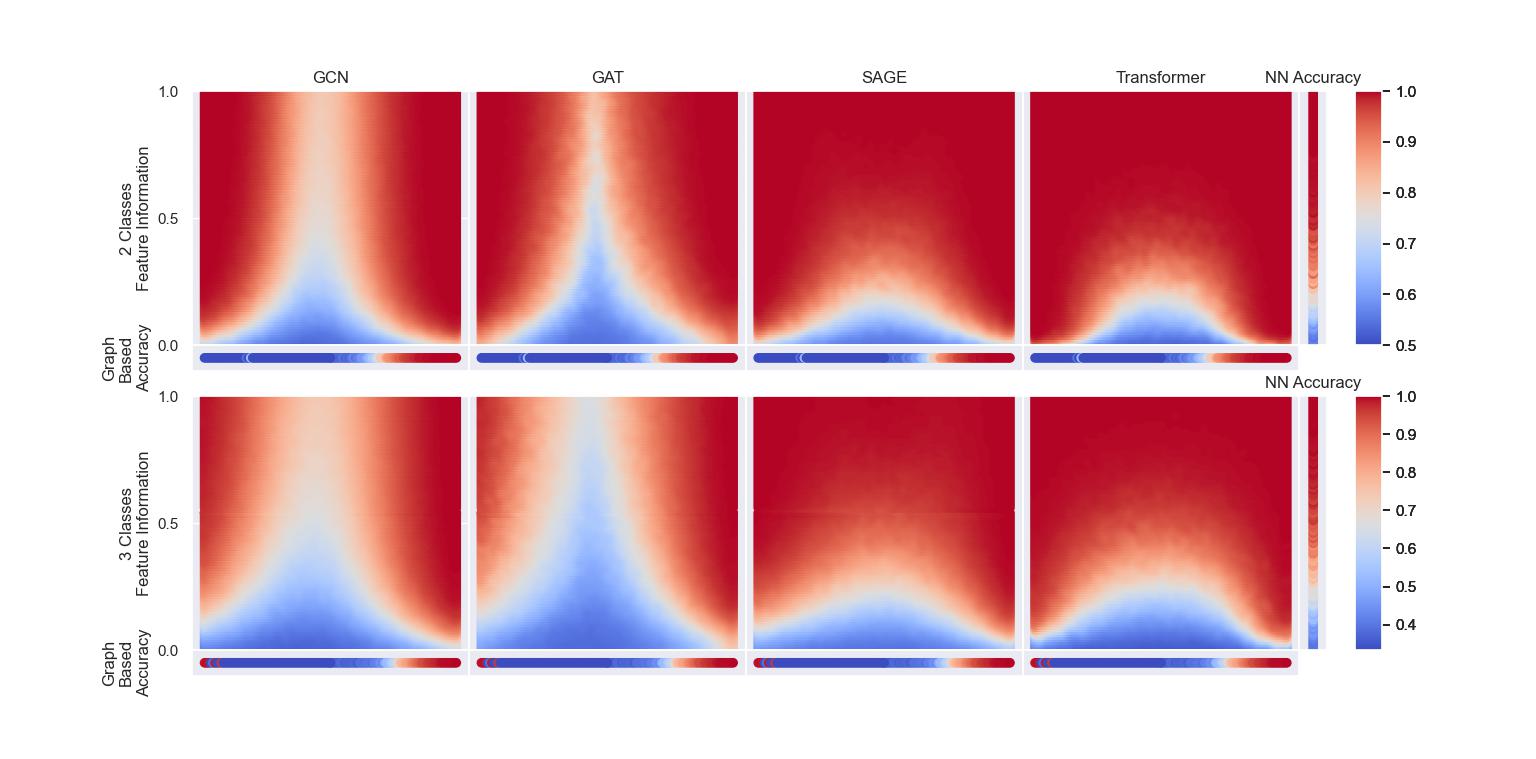}
    
    \includegraphics[width=\textwidth]{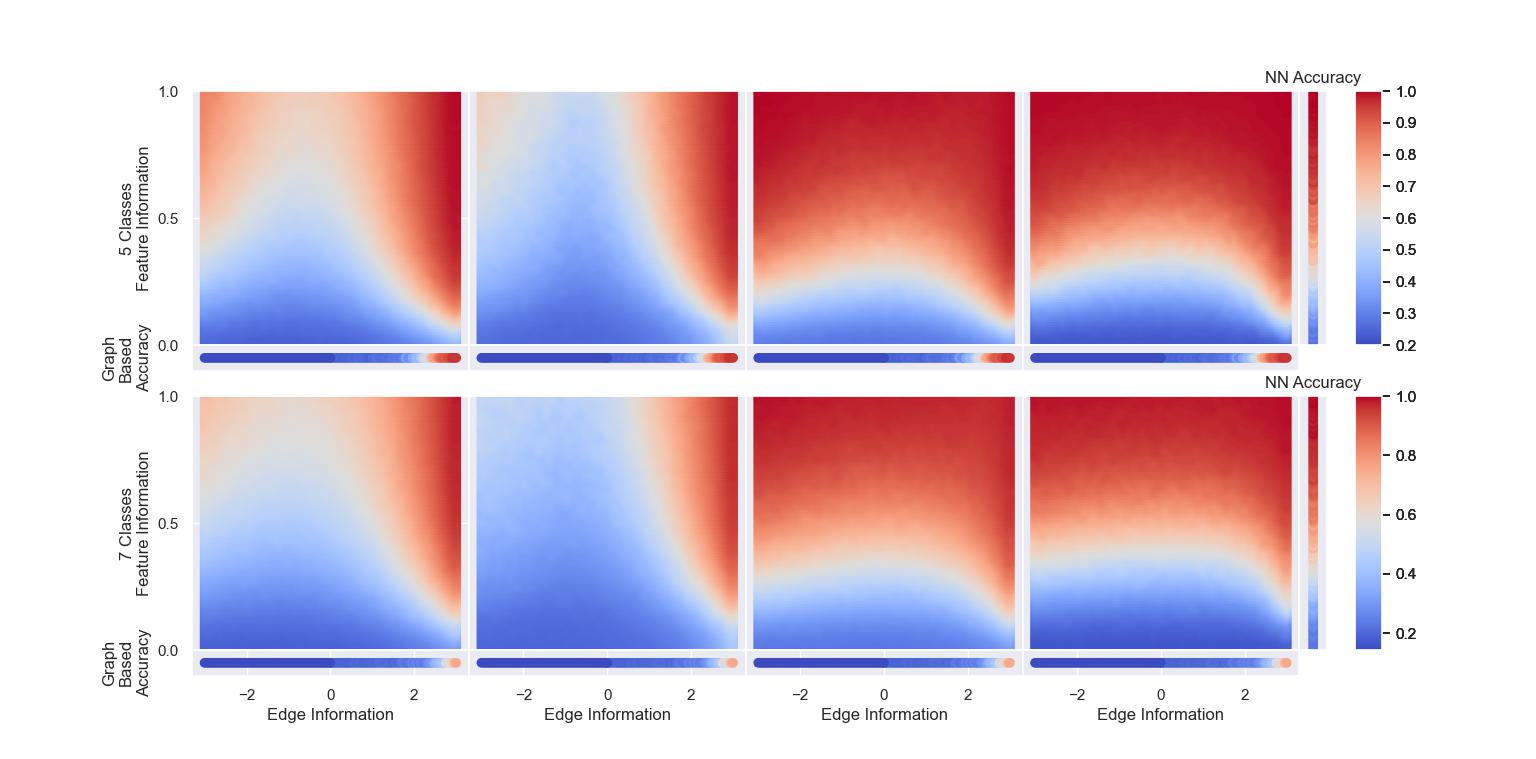}
    \caption{We compare accuracies over the distributed graphs across varying class sizes. Note that the blue regions of these graphs are much more pointed than those of the mean graphs.}
\end{figure}
When we view the maxes in light of the average graphs, we see that the blue portions of the maxes are much more steeply shaped than that of the averaged. This likely demonstrates that while both the averages and the maxes perform poorly towards the middle (where there is a lot of edge noise) the model is able to achieve better along the sides of the graph where we have less feature information but more edge information.
\begin{figure}[htbp]
    \label{fig:dc_max_panels}
    \centering
    \textbf{Degree-Corrected Degree Distribution}\par\medskip

    \includegraphics[width=\textwidth]{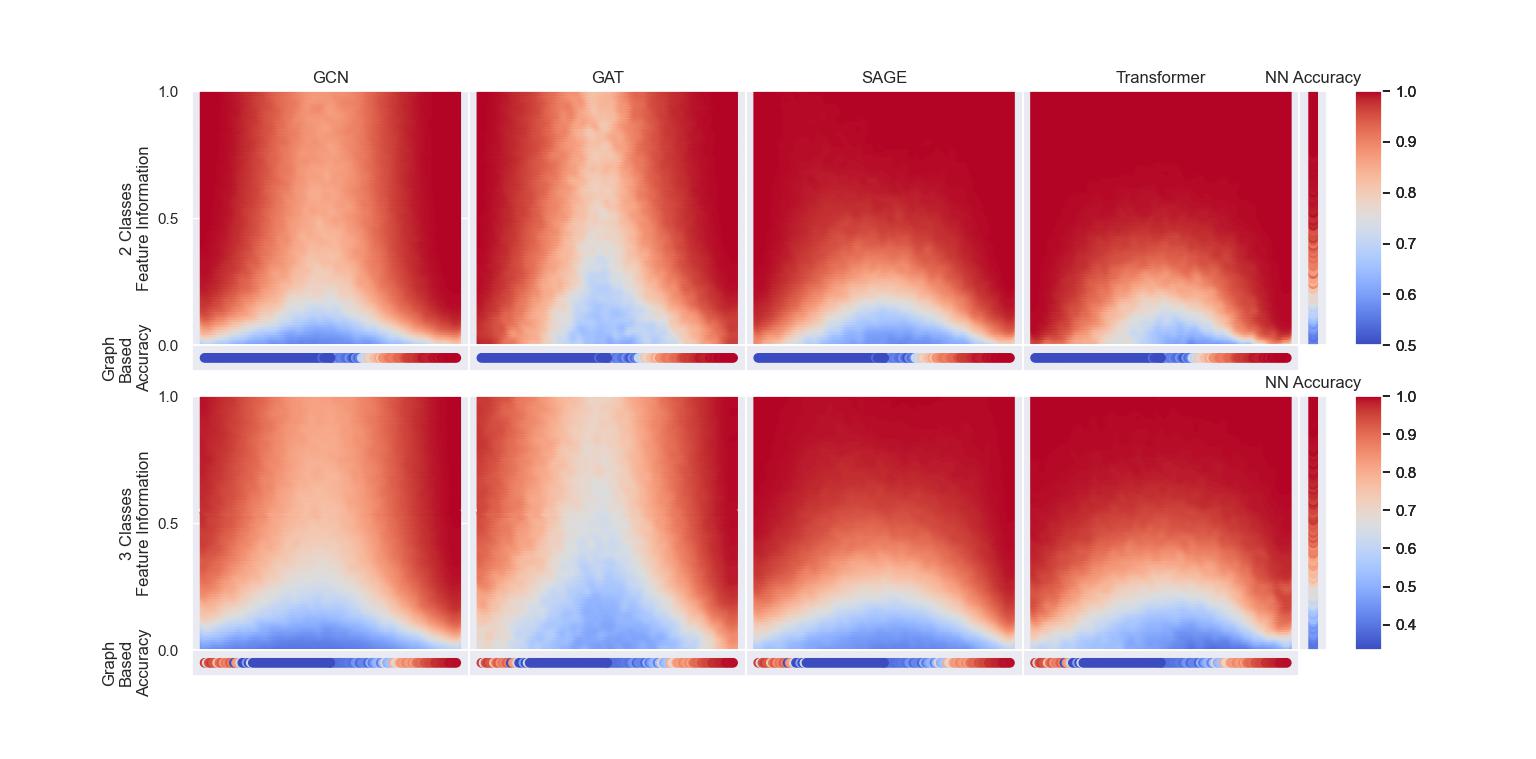}
    
    \includegraphics[width=\textwidth]{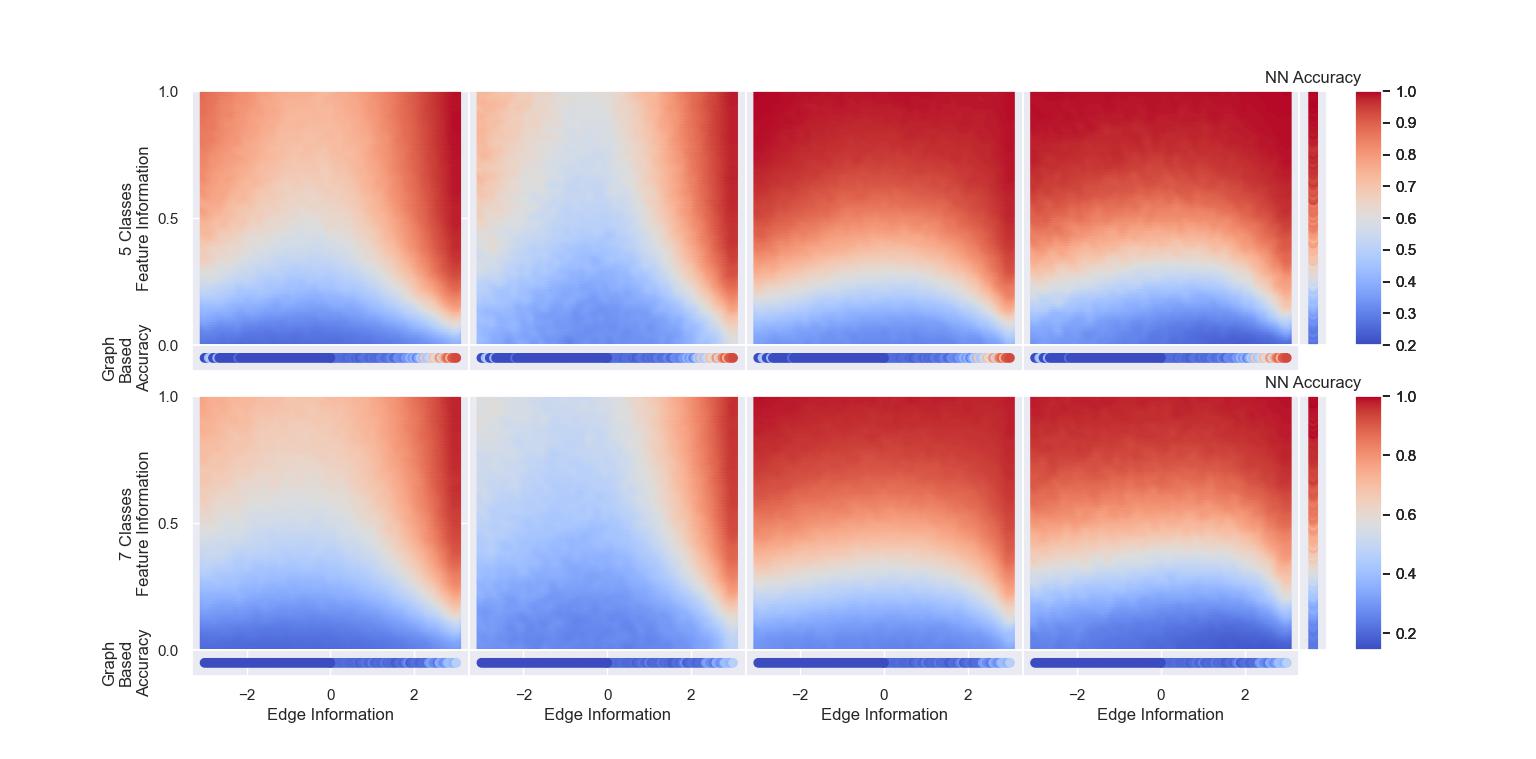}
    \caption{We compare accuracies over degree-corrected graphs across varying class sizes.}
\end{figure}
In general it seems that the models benefited from operating on heavy-tailed graphs. In particular we see that the GCN and the GAT performed better on degree corrected graphs across all class sizes. The Transformer and SAGE did not see as stark of an increase in performance, but did perform noticeably better on class sizes of 2 and 3.

\begin{figure}[htbp]
    \centering
    \textbf{Binomial Degree Distribution}\par\medskip

    \includegraphics[width=\textwidth]{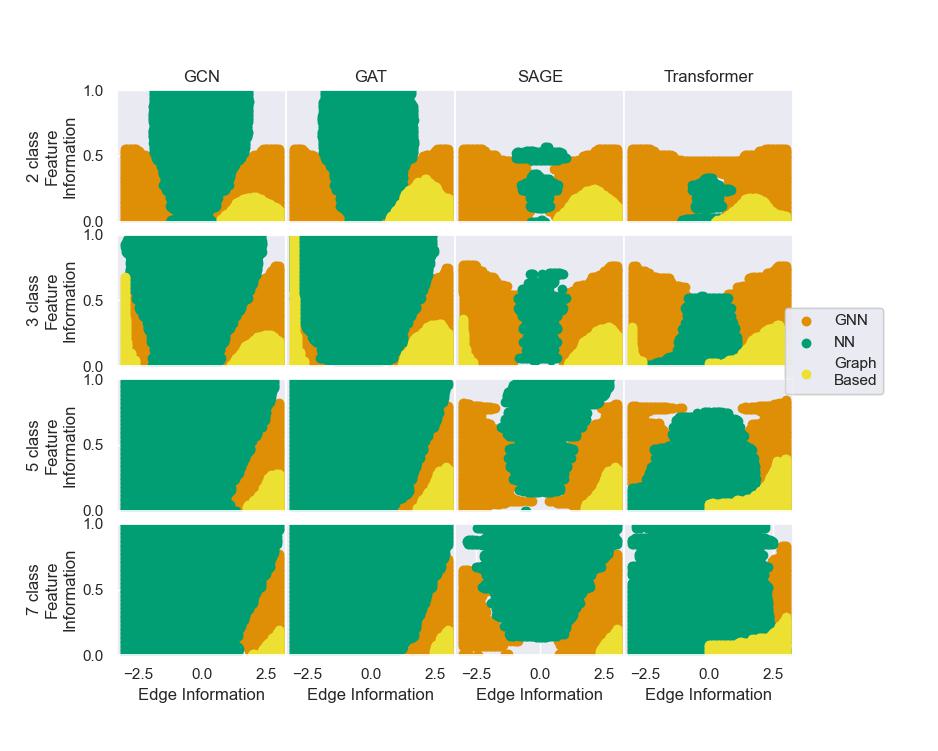}
    \caption{Regions of where each architecture outperforms the others across the SBM graphs. Here we can compare how varying parts of the data effects the shape and sizes of the favorable regimes}
    \label{fig:max_regions}
\end{figure}
\begin{figure}[htbp]
    \centering
    \textbf{Degree-Corrected Degree Distribution}\par\medskip

    \includegraphics[width=\textwidth]{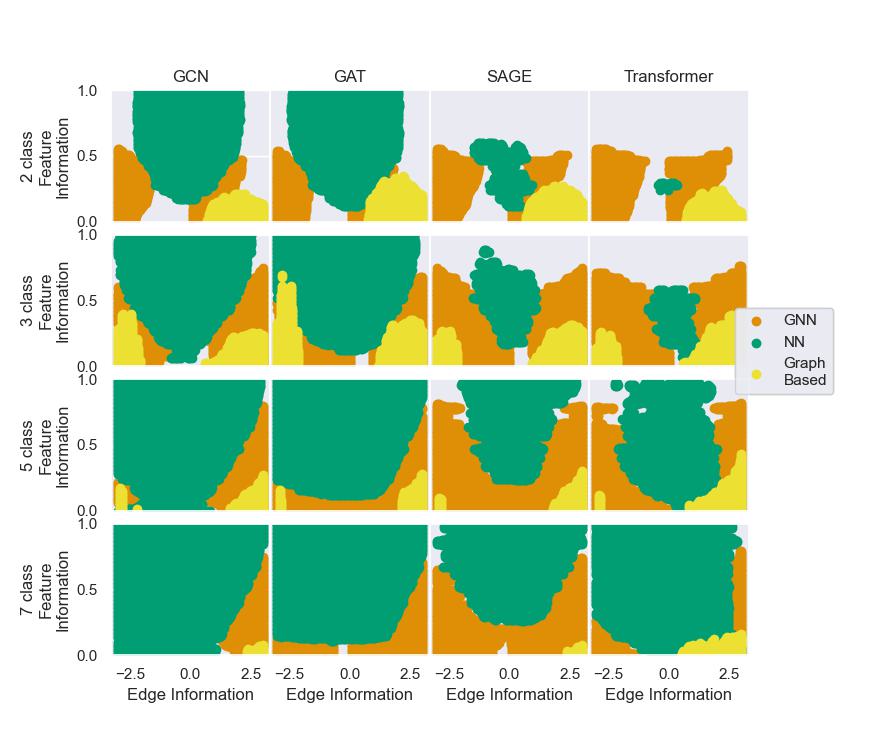}
    \caption{Regions where each architecture outperforms the others across degree-corrected graphs.}
    \label{fig:dc_max_regions}
\end{figure}

\section{Further analysis of the role of higher order structure}
\label{appendix:higher_order}

We provide further insights into the effects of higher order structure in GNNs. To illustrate the impact of structure, we develop several variants of attributed Stochastic Block Models. We track accuracy on a Hierarchical Stochastic Block Model (hSBM), a Epsilon Nearest Neighbors Stochastic Block Model (ENN-SBM), and a Stochastic Block Model with Triadic Closure. 

The implementations can be found in our code on GitHub. To create the hSBM, we generate 5 sub clusters for each class in the SBM that have slightly more similar features and self higher connectivity. We generate an Epsilon Nearest Neighbor Graph by sampling 1000 points from the unit square and randomly assigning half to each class. We then generate the edges by adding an edge between nodes if $||\text{node}_i-\text{node}_j||\leq \epsilon_{\text{intra}}$ for nodes of the same class for nodes of the same class and $||\text{node}_i-\text{node}_j||\leq \epsilon_{\text{inter}}$ for nodes of different classes. Lastly, we generate a triadic closed SBM by taking a normal SBM and closing $30\%$ of the possible triadic closures.

\begin{figure}
    \centering
    \includegraphics[width=\textwidth]{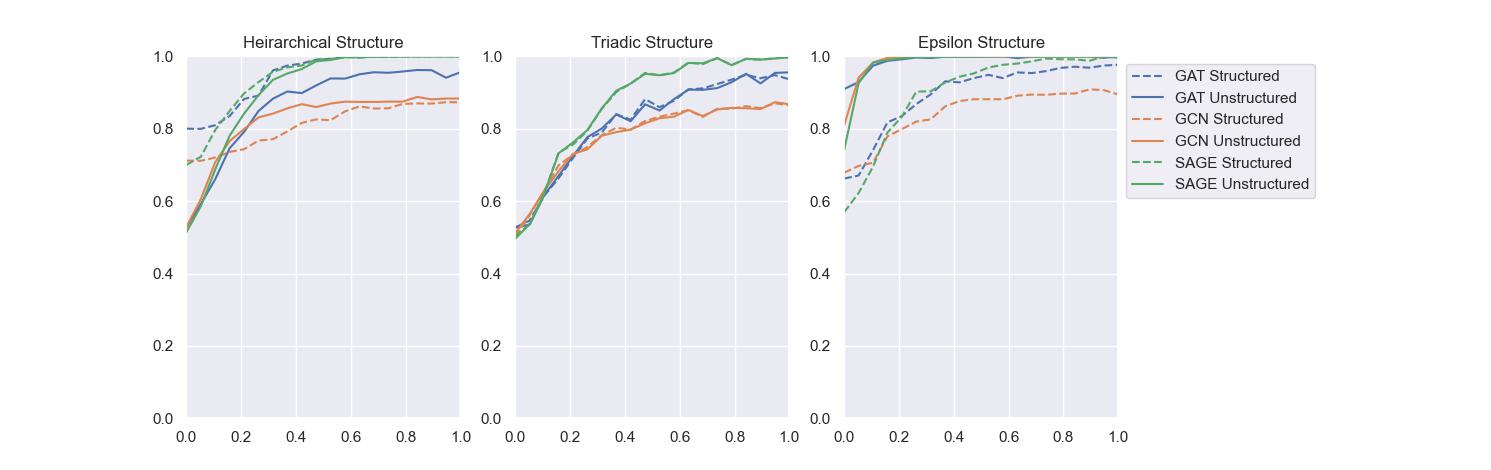}
    \caption{Effects of removing higher order structure in structured SBMs. We vary feature separability in each of the examples from 0 to 1. GNNs most notable increased under ENN-SBMs}
    \label{fig:higher_order}
\end{figure}
We note that as seen in \cref{fig:higher_order}
GNNs perform best on graphs lacking both geographic structure (encoded by ENN-SBM). One possible reason for this is that rewiring the graphs reduces the diameter of a graph, encouraging nodes to be closer to the center of the graph or their own communities. We also note that triadic closure has virtually no effect on the performance of GNNs, while results for hierarchical structure vary across architectures.

Hierarchical structure and spatial structure can both be seen as a version of label noise as a perfect graph might only connect groups that are relevant to one another. This is not true in all cases as often both of these attributes can contribute valuable information to a machine learning process. 

\end{document}